%% file: 00main.tex
\documentclass[]{fairmeta_gr2}

\usepackage{amsmath,amssymb,amsthm}
\usepackage{algorithm,algpseudocode}

\newtheorem{proposition}{Proposition}[section]

\titleformat{\paragraph}[runin]
  {\normalfont\normalsize\sffamily\bfseries}{}{0pt}{$\blacktriangleright$\enspace}
\titlespacing*{\paragraph}{0pt}{3.25ex plus 1ex minus .2ex}{0.5em}

\newif\ifshowcomments
\showcommentstrue

\ifshowcomments
  \newcommand{\zhe}[1]{{\small\textcolor{orange}{\bfseries zhe: #1}}}
  \newcommand{\todo}[1]{{\textcolor{red}{\bfseries TODO: #1}}}
\else
  \newcommand{\zhe}[1]{}
  \newcommand{\todo}[1]{}
\fi

\title{Exploring the Design Space of Permutation Diffusion LLMs}

\author[1]{Zhe Xu}
\author[1]{Mingfu Liang}
\author[1,*]{Kavosh Asadi}
\author[1,*]{Jean-Pascal Billaud}
\author[1,*]{Tianyi Chen}
\author[1,*]{Yanhui Guo}
\author[1,*]{Christina Huang}
\author[1,*]{Baekjin Kim}
\author[1,*]{Jimmy Kim}
\author[1,*]{Yufei Li}
\author[1,*]{Daqi Sun}
\author[1,*]{Jiachen Sun}
\author[1,*]{Dian Yu}
\author[1,*]{Zaiwei Zhang}
\author[1]{Kaushik Rangadurai}
\author[1]{Frank Shyu}
\author[1]{Sandeep Pandey}
\author[1]{Xi Liu}
\author[1]{Luke Simon}

\affiliation[1]{Meta}

\contribution[*]{Listed in alphabetical order by surname.}

\abstract{\input{files/01_abstract}}

\date{\today}
\correspondence{\email{zhexu@meta.com}, \email{xliu1@meta.com}}

\begin{document}

\maketitle

\input{files/02_introduction}
\input{files/03_related_work}
\input{files/04_preliminaries}
\input{files/05_method}
\input{files/06_experiments}
\input{files/07_conclusion}

\bibliographystyle{assets/plainnat}
\bibliography{ref}

\clearpage

\appendix
\input{files/08_appendix}

\section{Analysis}
\input{files/09_hard_sampling_analysis.tex}
\input{files/09_soft_sampling_analysis.tex}

\end{document}

%% file: files/01_abstract.tex
Permutation generation with language models requires both text-conditioned ranking and a global one-to-one assignment between candidates and output positions. We study this problem through three diffusion-LLM parameterizations: hard-assignment prediction over shuffled candidate-index sequences, soft-assignment prediction over doubly stochastic matrices, and direct ordering-score prediction followed by sorting. The three families place permutation constraints at different stages of prediction and decoding and therefore support different training objectives and sampling procedures. We evaluate these designs on recommendation reranking across three Amazon domains with candidate sets of size 10 and 20. At one function evaluation, the hard- and soft-assignment variants achieve Recall@1 comparable to a constrained masked-diffusion baseline, whereas direct ordering-score prediction guarantees valid permutations without Hungarian assignment or Sinkhorn normalization. Additional function evaluations do not consistently improve ranking quality, and their effect depends strongly on the state representation and training objective. These results show that corruption, prediction parameterization, objective, and decoder should be designed jointly for permutation diffusion with language-model denoisers.

%% file: files/02_introduction.tex

\section{Introduction}
\label{sec:introduction}

Permutation generation maps a fixed candidate set to an ordered sequence in which each candidate appears exactly once. This structure arises in ranking, sorting, and other combinatorial prediction tasks~\citep{vinyals2015pointer,zhang2025symmetricdiffusers}. Large language models (LLMs) can condition such orderings on textual candidate descriptions and task context, as demonstrated in listwise reranking~\citep{sun2023rankgpt}. Unlike unconstrained text generation, permutation generation requires a global one-to-one assignment between candidates and output positions.

Autoregressive (AR) LLMs generate candidate-index tokens sequentially and therefore require $K$ prediction steps for a set of size $K$. Diffusion LLMs (dLLMs) instead predict multiple positions in parallel and iteratively update the resulting state through successive function evaluations (NFEs)~\citep{sahoo2024mdlm,nie2025llada}. Recent work applies this paradigm to document and recommendation reranking~\citep{liu2026diffurank,zhang2026diffusiongr2}. Parallel prediction alone, however, does not guarantee a valid permutation because independently predicted identifiers may repeat or omit candidates. The central modeling challenge is thus to combine parallel prediction with the global assignment constraint.

Permutation-structured corruption provides a natural starting point. Hard shuffling changes candidate order while preserving permutation validity. A continuous alternative represents assignments as doubly stochastic matrices in the Birkhoff polytope, the convex hull of permutation matrices. SymmetricDiffusers, Soft-Rank Diffusion, and PermFlow study discrete or continuous permutation generation with task-specific neural architectures rather than pretrained general-purpose LLMs~\citep{zhang2025symmetricdiffusers,he2026softrank,min2026permflow}. We investigate how such structured states and prediction parameterizations can be integrated with a pretrained LLM for text-conditioned generation. Structural preservation during corruption is insufficient by itself: the prediction parameterization, training objective, and decoder jointly determine whether the model learns and returns valid assignments.

We study three parameterizations within a shared conditional clean-state prediction framework (Section~\ref{sec:method} and Table~\ref{tab:method-design-space}). \textbf{Hard-assignment diffusion LLMs (HA-dLLMs)} operate on shuffled candidate-index sequences and predict candidate scores for each output position; Hungarian assignment produces a permutation at every denoising step. \textbf{Soft-assignment diffusion LLMs (SA-dLLMs)} use continuous assignment matrices as both states and predictions; Sinkhorn normalization approximately enforces doubly stochastic constraints, and a terminal Hungarian assignment returns a discrete permutation. \textbf{Ordering-score diffusion LLMs (OS-dLLMs)} operate on shuffled sequences but predict one scalar score per candidate; sorting yields a valid permutation without Hungarian assignment or Sinkhorn normalization. The three families therefore impose the assignment constraint at different stages of prediction and decoding.

For each parameterization, we examine the corruption process, training objective, and sampling procedure. For HA-dLLMs, we compare token cross-entropy, row-plus-column assignment cross-entropy, Sinkhorn marginal negative log-likelihood, and a structured hinge loss. For SA-dLLMs, anchored-rollout training supervises predictions from model-generated states, and we compare Predict--recorrupt and Residual-correction sampling adapted from Cold Diffusion~\citep{bansal2023cold}. OS-dLLM combines a listwise ordering objective with sorting-based decoding.

We evaluate these configurations on recommendation reranking across three Amazon domains with candidate sets of 10 and 20 items. The task reorders a supplied candidate set rather than generating a single item. At one function evaluation (NFE$=1$), the hard- and soft-assignment variants achieve reranking quality comparable to a constrained masked diffusion language model (mdLLM). Direct ordering-score prediction uniquely guarantees permutation validity without Hungarian assignment or Sinkhorn normalization. The experiments characterize how the design choices affect ranking quality, validity, decoding structure, and sensitivity to NFE.

Our contributions are threefold:
\begin{itemize}
    \item We formulate three permutation-diffusion LLM families---hard assignment, soft assignment, and direct ordering score---and distinguish their state representations and validity mechanisms.
    \item We develop family-specific corruption processes\footnote{We use \emph{corruption}, \emph{noising}, and \emph{forward process} interchangeably; likewise, \emph{corruption level} and \emph{diffusion time} are synonymous. \emph{Denoising} and \emph{sampling} refer to reverse generation.}, objectives, and samplers, including structured assignment losses and model-generated soft-state training.
    \item We empirically characterize the resulting tradeoffs in ranking quality, permutation validity, and iterative behavior on recommendation reranking.
\end{itemize}

%% file: files/03_related_work.tex

\section{Related Work}
\label{sec:related-work}

\paragraph{Diffusion models for discrete structures.} 
Discrete denoising diffusion extends generative modeling beyond continuous Gaussian corruption by specifying transition processes over categorical states~\citep{austin2021d3pm}. Masked diffusion specializes this approach to absorbing mask states, with objectives closely connected to masked language modeling~\citep{sahoo2024mdlm}. Subsequent work scales this formulation to instruction-following language models~\citep{nie2025llada}. Beyond sequences, DISCO formulates discrete-state continuous-time diffusion for graph generation and establishes permutation-equivariant and permutation-invariant properties with respect to node ordering~\citep{xu2024disco}.

\paragraph{Cold Diffusion.}
Cold Diffusion studies image generation under general deterministic degradations rather than discrete-sequence diffusion~\citep{bansal2023cold}. Its restoration-and-recorruption perspective motivates our samplers. We adapt this perspective to assignments: hard corruption rearranges a fixed candidate set, whereas soft corruption interpolates between doubly stochastic matrices.

\paragraph{Generative models of permutations.}
SymmetricDiffusers defines discrete diffusion on finite symmetric groups using shuffle-based forward processes and generalized Plackett--Luce reverse parameterizations~\citep{zhang2025symmetricdiffusers}. Soft-Rank Diffusion constructs corruption through continuous soft ranks and uses contextualized generalized Plackett--Luce denoisers to predict clean permutations~\citep{he2026softrank}. PermFlow learns conditional flows on the affine space of matrices with unit row and column sums, targeting multimodal permutation distributions~\citep{min2026permflow}. These works evaluate task-specific permutation architectures rather than pretrained general-purpose LLMs. We incorporate related structural ideas into a pretrained language-model backbone and compare per-step assignment projection, soft-state recurrence, and direct ordering-score prediction. Unlike soft-rank vectors or unconstrained affine states, SA-dLLM evolves nonnegative doubly stochastic matrices.

\paragraph{Assignment relaxations and neural ordering.}
Pointer Networks established neural sequence prediction over input-dependent object sets~\citep{vinyals2015pointer}. Listwise maximum likelihood estimation (ListMLE) defines a likelihood objective for ordering scores~\citep{xia2008listmle}, which we apply to the source-slot order induced by a corrupted permutation. For matrix-valued assignments, Sinkhorn normalization enforces row and column marginals through iterative scaling~\citep{sinkhorn1967matrices}, while Gumbel--Sinkhorn Networks use it to construct differentiable relaxations of latent permutations~\citep{mena2018gumbelsinkhorn}. Hungarian assignment provides the discrete optimization primitive for our projected decoders~\citep{kuhn1955hungarian}. We compare how these components impose assignment constraints during language-model sampling.

\paragraph{Language models for reranking and recommendation.}
RankGPT frames listwise reranking as instructional permutation generation~\citep{sun2023rankgpt}, while RankZephyr develops specialized open listwise rerankers~\citep{pradeep2023rankzephyr}. Faster Improved Listwise Reranking with Single Token Decoding (FIRST) derives a complete ranking from first-token logits, showing that LLM reranking need not generate a full autoregressive identifier sequence~\citep{reddy2024first}. DiffuRank studies masked-diffusion reranking with parallel scoring, constrained permutation sampling, and Hungarian assignment at NFE$=1$~\citep{liu2026diffurank}. Our constrained mdLLM baseline is closely related to this assignment formulation; HA-dLLM and OS-dLLM replace masking with shuffling, whereas SA-dLLM evolves an assignment matrix throughout inference. Recent work also advocates discriminative LLM retrieval rather than generative retrieval for efficient candidate generation~\citep{xu2026case}; our setting begins after retrieval and ranks a fixed candidate set. In recommendation, GR2 introduces an autoregressive generative reasoning reranker trained through semantic-ID mid-training, reasoning-trace distillation, and reinforcement learning~\citep{liang2026gr2}. Its technical report extends the system with context compression, on-policy distillation, and serving-oriented reasoning distillation~\citep{li2026gr2technical}. Diffusion-GR2 then converts GR2 to block diffusion through conversion fine-tuning, on-policy distillation, and reinforcement learning~\citep{zhang2026diffusiongr2}. We instead model the permutation state directly, enforce validity through sorting or assignment, and generate no explicit reasoning trace. LLaDA-Rec applies masked diffusion to parallel semantic-identifier generation~\citep{shi2025lladarec}; its next-item generation task differs from reordering all members of a supplied candidate set.

\paragraph{Training on model-generated states.}
Conditioning on model-generated states creates distribution shift in sequential prediction. Dataset Aggregation (DAgger) addresses this problem by collecting supervision on learner-induced states~\citep{ross2011dagger}, and Diffusion-GR2 applies on-policy teacher supervision to generated reasoning-and-ranking trajectories~\citep{zhang2026diffusiongr2}. Our anchored-rollout objective extends this motivation to continuous assignment states. We measure stability across NFE budgets rather than assume that additional evaluations improve quality.

%% file: files/04_preliminaries.tex
\section{Preliminaries}
\label{sec:preliminaries}

\subsection{Notation and Problem Formulation}
\label{sec:prelim-notation}

Plain italic symbols denote scalars (e.g., $K,N,t$), bold lowercase symbols denote numerical vectors (e.g., $\mathbf{h}_i$), and bold uppercase symbols denote matrices (e.g., $\mathbf{S},\mathbf{X}_t$). Calligraphic symbols denote sets and spaces (e.g., $\mathcal{V}_{\mathrm{idx}},\mathcal{B}_K$). Permutations, symbolic token sequences, and the parameter collection $\theta$ remain unbolded.

For a matrix $\mathbf{A}$, we write $A_{ij}$ for scalar entries and $\mathbf{A}_{i,:}$ for rows; for time-indexed matrices, $X_{t,ij}=(\mathbf{X}_t)_{ij}$. Scalar entries remain unbolded, including those of predicted matrices such as $Q^{(n)}_{ij}$. We apply $\log \mathbf{A}$ and $\max\{\mathbf{A},\varepsilon\}$ elementwise. All $\operatorname*{arg\,max}$ expressions select one maximizer using a fixed tie-breaking rule. Permutations compose as $(a\circ b)(i)=a(b(i))$.

Let $K\geq2$ be the number of candidates and $[K]=\{1,\ldots,K\}$. The context $c$ contains the user context and the presented candidate descriptions $x_{1:K}$. Let $\mathcal{V}_{\mathrm{raw}}$ denote the pretrained vocabulary and $\mathcal{V}_{\mathrm{idx}}=\{\nu_1,\ldots,\nu_K\}$ the added candidate-index tokens, with $\nu_j$ serialized as \texttt{<idx\_j>}. The expanded vocabulary $\mathcal{V}$ contains these sets and additional task-specific tokens. Index $j$ identifies a candidate in the presented list rather than an output position: token \texttt{<idx\_j>} at output position $i$ places candidate $j$ at rank $i$. We write $\theta$ for the model parameters and $\mathfrak{S}_K$ for the symmetric group on $[K]$. A ranking is a permutation $\pi\in\mathfrak{S}_K$, where $\pi(i)=j$ means that candidate $j$ occupies output position $i$. The clean target permutation $\pi_0$ corresponds to the token sequence
\begin{equation}
  y_0
  =\left(\nu_{\pi_0(1)},\ldots,\nu_{\pi_0(K)}\right).
  \label{eq:prelim-clean-token-sequence}
\end{equation}
Candidate indices and output positions are one-based. The clean target has permutation matrix
\begin{equation}
  P_{0,ij}=\mathbf{1}\!\left[j=\pi_0(i)\right].
  \label{eq:prelim-clean-permutation}
\end{equation}

\subsection{Autoregressive Reranking}
\label{sec:prelim-ar}

An autoregressive reranker predicts the candidate-index sequence $y_0$ from left to right~\citep{vinyals2015pointer,sun2023rankgpt}. Let $p_\theta$ be the vocabulary distribution. For candidate-token prediction, we use its restriction to $\mathcal{V}_{\mathrm{idx}}$,
\begin{equation}
  p_\theta^{\mathrm{idx}}(v\mid c,y_{<i})
  =\frac{p_\theta(v\mid c,y_{<i})}
    {\sum_{v'\in\mathcal{V}_{\mathrm{idx}}}p_\theta(v'\mid c,y_{<i})},
  \qquad v\in\mathcal{V}_{\mathrm{idx}},
  \label{eq:prelim-index-distribution}
\end{equation}
where $y_{<i}$ denotes the preceding candidate tokens. The candidate-sequence distribution factorizes as
\begin{equation}
  p_{\theta}^{\mathrm{AR}}(y_0\mid c)
  =\prod_{i=1}^{K}
    p_\theta^{\mathrm{idx}}\!\left(
      y_{0,i}
      \mid c,y_{0,<i}
    \right).
  \label{eq:prelim-ar-factorization}
\end{equation}
The candidate-token loss under teacher-forced causal fine-tuning is
\begin{equation}
  \mathcal{L}_{\mathrm{AR}}
  =-\frac{1}{K}
    \sum_{i=1}^{K}
    \log p_\theta^{\mathrm{idx}}\!\left(
      y_{0,i}
      \mid c,y_{0,<i}
    \right).
  \label{eq:prelim-ar-loss}
\end{equation}

At inference, the evaluated AR baseline uses permutation-constrained greedy decoding:
\begin{equation}
  \widehat y_i
  =\operatorname*{arg\,max}_{v\in\mathcal{V}_{\mathrm{idx}}\setminus\{\widehat y_1,\ldots,\widehat y_{i-1}\}}
    p_\theta^{\mathrm{idx}}\!\left(v\mid c,\widehat y_{<i}\right).
  \label{eq:prelim-ar-greedy}
\end{equation}
Restricting each prediction to unused candidate-index tokens ensures that every candidate appears exactly once; unconstrained greedy decoding may repeat or omit candidates. Generating the full ranking requires $K$ sequential prediction steps. Section~\ref{sec:experiments} specifies the output format.

\subsection{Masked Diffusion Language-Model Reranking}
\label{sec:prelim-mdllm}

The masked diffusion language model (mdLLM) baseline adapts masked clean-token prediction~\citep{sahoo2024mdlm,nie2025llada} to candidate-index sequences.

\paragraph{Candidate scores.}
For parallel prediction, let $L_{iv}$ be the logit for vocabulary token $v\in\mathcal{V}$ at output position $i$. The candidate-score matrix is
\begin{equation}
  S_{ij}=L_{i\nu_j},
  \qquad \mathbf{S}\in\mathbb{R}^{K\times K},
  \label{eq:prelim-candidate-scores}
\end{equation}
Rows index output positions and columns index candidates. During iterative inference, $\mathbf{S}^{(n)}$ denotes the score matrix at step $n$, with steps indexed as $n=N,N-1,\ldots,1$ for an NFE budget of $N$. Each function evaluation scores all $K$ positions in parallel.

\subsubsection{Forward masking and clean-state prediction}

The forward process replaces a random subset of clean ranking tokens with the mask token $\nu_{\mathrm{mask}}$, serialized as \texttt{<rank\_mask>}. We draw $\rho\sim\operatorname{Unif}(\rho_{\min},1)$, where $\rho_{\min}\in[0,1)$, and independently sample $\beta_i\sim\operatorname{Bernoulli}(\rho)$ for each $i\in[K]$. If no position is masked, we select one uniformly and set its indicator to one. The resulting mask set and corrupted sequence are
\begin{equation}
  \Omega=\{i\in[K]:\beta_i=1\},
  \label{eq:prelim-md-mask-set}
\end{equation}
\begin{equation}
  \widetilde y_i=
  \begin{cases}
    \nu_{\mathrm{mask}}, & i\in\Omega,\\
    y_{0,i}, & i\notin\Omega.
  \end{cases}
  \label{eq:prelim-md-forward}
\end{equation}

The context is encoded as a causal prefix, while the $K$ ranking slots attend to the prefix and bidirectionally to one another. Prediction is unshifted: the hidden state at position $i$ predicts the clean token $y_{0,i}$. Using the candidate-score matrix $\mathbf{S}$ from Equation~\eqref{eq:prelim-candidate-scores}, the single-example cross-entropy over masked positions is
\begin{equation}
  \mathcal{L}_{\mathrm{mdLLM}}
  =-\frac{1}{|\Omega|}
    \sum_{i\in\Omega}
    \log
    \frac{\exp S_{i,\pi_0(i)}}
         {\sum_{j=1}^{K}\exp S_{ij}}.
  \label{eq:prelim-md-loss}
\end{equation}

\subsubsection{Raw confidence decoding}

We adapt confidence-based parallel decoding~\citep{chang2022maskgit,nie2025llada} to fixed-length candidate sequences. Inference begins from $K$ masked positions and uses $1\leq N\leq K$ function evaluations, indexed in decreasing order as $n=N,N-1,\ldots,1$. Each evaluation predicts the candidate-score matrix $\mathbf{S}^{(n)}$ in parallel. Let $\mathcal{U}_n\subseteq[K]$ denote the unresolved positions, initially $\mathcal{U}_N=[K]$. Their confidence scores are
\begin{equation}
  \alpha_i^{(n)}
  =\max_{j\in[K]}
    \frac{\exp S^{(n)}_{ij}}
         {\sum_{j'=1}^{K}\exp S^{(n)}_{ij'}},
  \qquad i\in\mathcal{U}_n.
  \label{eq:prelim-md-confidence}
\end{equation}
Within each row, the largest confidence corresponds to the largest logit.


To fill all positions in exactly $N$ evaluations, step $n$ commits
\begin{equation}
  m_n
  =\left\lceil
    \frac{|\mathcal{U}_n|}{n}
   \right\rceil
  \label{eq:prelim-md-commit-count}
\end{equation}
positions with the largest confidence. Denote this selected set by $\mathcal{I}_n\subseteq\mathcal{U}_n$. The raw decoder predicts
\begin{equation}
  \widehat j_i
  =\operatorname*{arg\,max}_{j\in[K]}S^{(n)}_{ij},
  \qquad i\in\mathcal{I}_n,
  \label{eq:prelim-md-raw-commit}
\end{equation}
and sets $\mathcal{U}_{n-1}=\mathcal{U}_n\setminus\mathcal{I}_n$. Committed tokens remain visible in subsequent evaluations. Because candidate choices are independent across positions, simultaneous commitments may repeat candidates, and later evaluations may reuse candidates committed earlier.

\subsection{Assignment Operators}
\label{sec:prelim-assignment}

\paragraph{Discrete assignment.}
Given a score matrix $\mathbf{S}\in\mathbb{R}^{K\times K}$, an assignment operator enforces a one-to-one correspondence between output positions and candidates. Define the set of permutation matrices as
\begin{equation}
  \mathcal{P}_K
  =\left\{
    \mathbf{P}\in\{0,1\}^{K\times K}:
    \mathbf{P}\mathbf{1}=\mathbf{1},\;
    \mathbf{P}^\top\mathbf{1}=\mathbf{1}
  \right\}
  \label{eq:prelim-permutation-matrices}
\end{equation}
where $\mathbf{1}\in\mathbb{R}^K$ is the all-ones vector. The maximum-score assignment is
\begin{equation}
  \widehat{\mathbf{P}}
  \in\operatorname*{arg\,max}_{\mathbf{P}\in\mathcal{P}_K}
    \langle \mathbf{S},\mathbf{P}\rangle,
  \qquad \langle \mathbf{S},\mathbf{P}\rangle=\sum_{i,j}S_{ij}P_{ij}.
  \label{eq:prelim-assignment}
\end{equation}
The Hungarian algorithm solves this linear assignment problem~\citep{kuhn1955hungarian}. Each row and column of $\widehat{\mathbf{P}}$ contains exactly one nonzero entry, yielding a valid permutation.

\paragraph{Soft assignment.}
Relaxing binary entries to nonnegative values yields the Birkhoff polytope of doubly stochastic matrices,
\begin{equation}
  \mathcal{B}_K
  =\left\{
    \mathbf{X}\in\mathbb{R}_{\geq 0}^{K\times K}:
    \mathbf{X}\mathbf{1}=\mathbf{1},\;
    \mathbf{X}^\top\mathbf{1}=\mathbf{1}
  \right\}.
  \label{eq:prelim-birkhoff}
\end{equation}
For temperature $\tau>0$, Sinkhorn normalization alternately rescales the rows and columns of the elementwise exponential $\exp(\mathbf{S}/\tau)$ to sum to one~\citep{sinkhorn1967matrices,mena2018gumbelsinkhorn}. We denote this operation by $\operatorname{SK}_{\tau}(\mathbf{S})$ and implement it in log space for numerical stability; Appendix~\ref{sec:implementation-details} specifies the stopping criterion and iteration limit. The output is an approximately doubly stochastic soft assignment rather than a discrete permutation.


\paragraph{Constrained mdLLM decoding.}
Assignment-based permutation decoding has also been studied for diffusion reranking~\citep{liu2026diffurank}. The constrained masked decoder applies discrete assignment to each selected subset of positions. Let $\mathcal{J}_n\subseteq[K]$ denote the unused candidates, with $\mathcal{J}_N=[K]$ initially. Confidence is recomputed over $\mathcal{J}_n$, and the $m_n$ most confident positions in $\mathcal{U}_n$ form $\mathcal{I}_n$. Their joint assignment is
\begin{equation}
  a_n^*
  =\operatorname*{arg\,max}_{
      a:\mathcal{I}_n\hookrightarrow\mathcal{J}_n
    }
    \sum_{i\in\mathcal{I}_n}S^{(n)}_{i,a(i)},
  \label{eq:prelim-md-constrained-commit}
\end{equation}
where $\hookrightarrow$ denotes an injective mapping. A rectangular Hungarian solve assigns distinct unused candidates to the selected positions. The selected positions are committed, and their candidates are removed from $\mathcal{J}_n$ before the next evaluation. At $N=1$, the procedure reduces to a global $K\times K$ assignment; at $N=K$, each evaluation commits one position and one unused candidate. It therefore returns a valid permutation for every supported budget, using one shrinking rectangular assignment solve per evaluation.

%% file: files/05_method.tex
\section{Method}
\label{sec:method}

\paragraph{Design-space overview.}
We consider three parameterizations of conditional clean-state prediction: hard-assignment prediction (HA-dLLM), soft-assignment prediction (SA-dLLM), and direct ordering-score prediction (OS-dLLM). Table~\ref{tab:method-design-space} summarizes their state representations, corruption processes, objectives, samplers, and decoders.


Figures~\ref{fig:permutation-diffusion-training} and~\ref{fig:permutation-diffusion-workflow} illustrate the shared training and sampling workflows. Training corrupts the target permutation while holding the instruction prefix fixed and optimizes the model to predict the clean state. At inference, each model evaluation predicts a clean state from the current corrupted state; a multi-step sampler then constructs the input to the next evaluation.

\begin{figure}[t!]
  \centering
  \includegraphics[width=.8\textwidth]{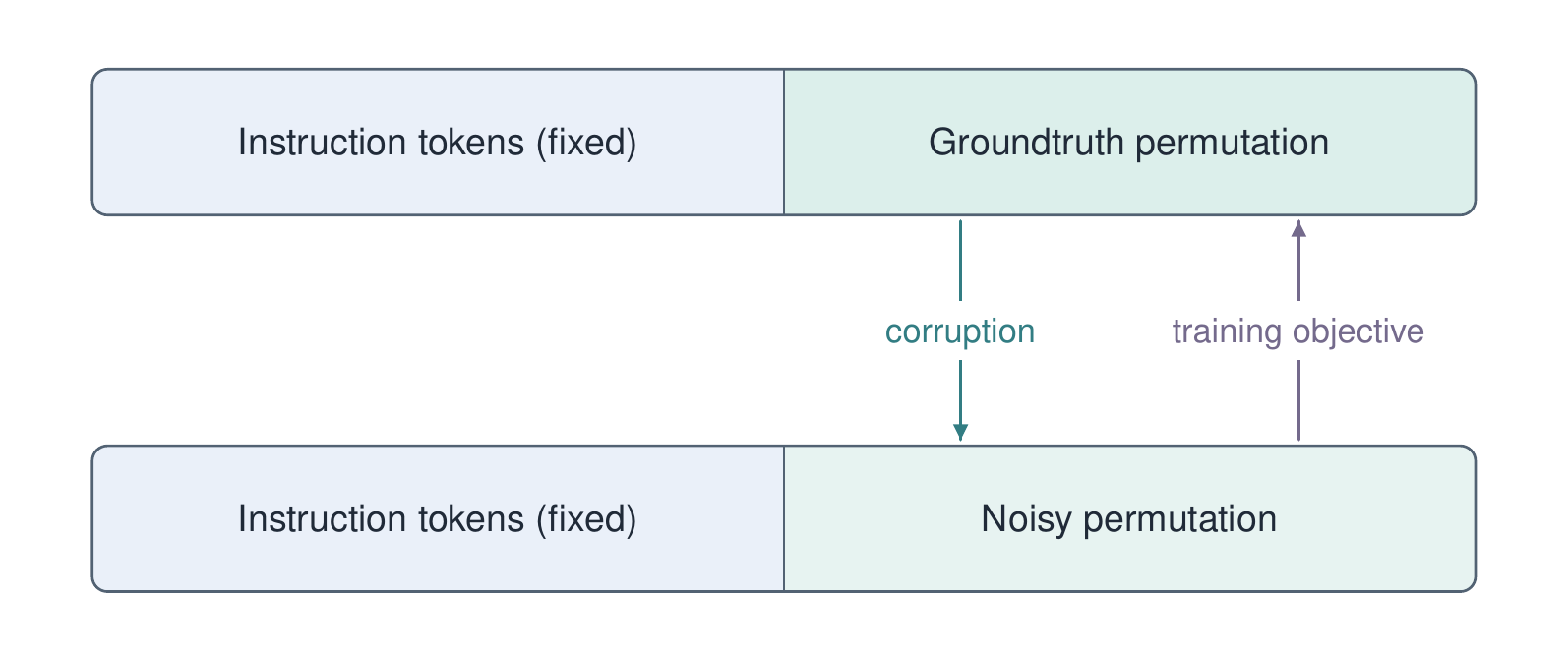}
  \caption{Permutation-diffusion training. The instruction prefix remains fixed while the target permutation is corrupted. The model predicts the clean permutation from the corrupted state and is optimized with a family-specific objective.}
  \label{fig:permutation-diffusion-training}
\end{figure}

\begin{figure}[t!]
  \centering
  \includegraphics[width=.9\textwidth]{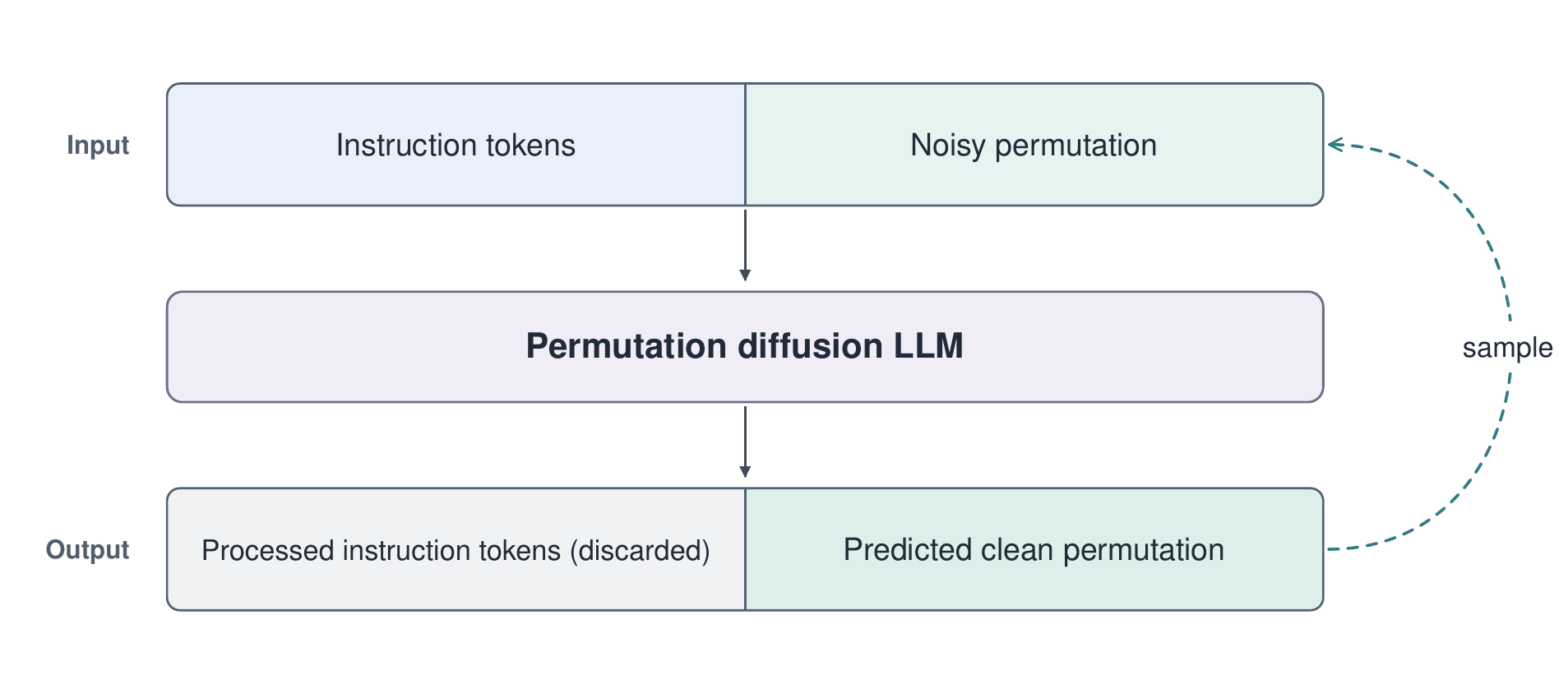}
  \caption{Permutation-diffusion sampling. The model conditions on the fixed instruction prefix and current corrupted state to predict the clean permutation. In multi-step sampling, this prediction determines the state for the next evaluation.}
  \label{fig:permutation-diffusion-workflow}
\end{figure}

\begin{table*}[!t]
  \centering
  \small
  \setlength{\tabcolsep}{4pt}
  \renewcommand{\arraystretch}{1.18}
  \caption{Design choices studied in this work. Losses and samplers listed within a family are alternatives. CE denotes cross-entropy and NLL denotes negative log-likelihood. Each model ultimately returns a discrete permutation.}
  \label{tab:method-design-space}
  \begin{tabular}{@{}p{0.16\textwidth}p{0.19\textwidth}p{0.24\textwidth}p{0.14\textwidth}p{0.19\textwidth}@{}}
    \toprule
    \raggedright Model family &
    \raggedright State and corruption &
    \raggedright Training loss options &
    \raggedright Sampling options &
    \raggedright Discrete decoding \tabularnewline
    \midrule
    \multirow{4}{=}{\raggedright Hard-assignment prediction (HA-dLLM)} &
    \multirow{4}{=}{\raggedright Hard permutation; subset shuffling} &
    \raggedright Token CE &
    \multirow{4}{=}{\raggedright Predict--recorrupt} &
    \multirow{4}{=}{\raggedright Hungarian assignment at every denoising step} \tabularnewline
    \cmidrule(lr){3-3}
    & & \raggedright Row-plus-column CE & & \tabularnewline
    \cmidrule(lr){3-3}
    & & \raggedright Sinkhorn marginal NLL & & \tabularnewline
    \cmidrule(lr){3-3}
    & & \raggedright Structured hinge & & \tabularnewline
    \midrule
    \multirow{4}{=}{\raggedright Soft-assignment prediction (SA-dLLM)} &
    \multirow{4}{=}{\raggedright Doubly stochastic matrix; affine soft corruption} &
    \multirow{4}{=}{\raggedright Anchored-rollout marginal loss} &
    \multirow{2}{=}{\raggedright Predict--recorrupt} &
    \multirow{4}{=}{\raggedright One terminal Hungarian assignment} \tabularnewline
    & & & & \tabularnewline
    \cmidrule(lr){4-4}
    & & & \multirow{2}{=}{\raggedright Residual-correction} & \tabularnewline
    & & & & \tabularnewline
    \midrule
    \raggedright Direct ordering-score prediction (OS-dLLM) &
    \raggedright Hard permutation; subset shuffling &
    \raggedright Plackett--Luce/ListMLE &
    \raggedright Predict--recorrupt &
    \raggedright Stable sorting at every denoising step; no assignment solver \tabularnewline
    \bottomrule
  \end{tabular}
\end{table*}


The algorithms below describe single-example updates; Sections~\ref{sec:experiments-models} and~\ref{sec:experiments-protocol} specify the backbone, initialization, and training protocol. The textual context forms a causal prefix followed by $K$ ranking slots with bidirectional attention among slots. HA-dLLM and SA-dLLM use unshifted prediction: the hidden state at destination slot $i$ predicts the clean assignment at that slot, unlike AR next-token prediction. OS-dLLM instead assigns an ordering score to the item in each current slot.

In mdLLM, the mask pattern reveals the corrupted positions and implicitly indicates the corruption level. Subset shuffling introduces no mask tokens, so the hard-state models receive $T$ as an additional input. A Fourier-feature MLP~\citep{tancik2020fourier} embeds the corruption level\footnote{We use \emph{corruption level} and \emph{diffusion time} interchangeably.} and adds the resulting vector to each ranking-slot embedding. Its output projection is initialized to zero, so the initial timestep embedding is zero.

\subsection{Hard-Assignment Prediction (HA-dLLM)}
\label{sec:method-token-assignment}

HA-dLLM represents every forward and reverse state as a discrete candidate-index sequence. The model predicts real-valued candidate scores; the term ``hard assignment'' describes the state representation and decoder, not the raw model output.

\subsubsection{Input and output representation}

At corruption level $T\in[K]$, the state $\pi_T$ is serialized as $(\nu_{\pi_T(1)},\ldots,\nu_{\pi_T(K)})$, and the model receives $(c,\pi_T,T)$. Training and inference use the same input representation. The Fourier-feature MLP receives the normalized level $T/K$. Let $d$ denote the common dimension of the token embeddings and backbone hidden states, and let $\mathbf{h}_i\in\mathbb{R}^{d}$ be the final-layer representation at ranking slot $i$. The vocabulary projection produces logits
\begin{equation}
  L_{iv}=\mathbf{w}_v^\top \mathbf{h}_i+b_v,
  \qquad v\in\mathcal{V}.
  \label{eq:method-vocabulary-logits}
\end{equation}
The resulting candidate-score matrix is $\mathbf{S}_\theta(c,\pi_T,T)\in\mathbb{R}^{K\times K}$, with entries $S_{ij}=L_{i\nu_j}$ as in Equation~\eqref{eq:prelim-candidate-scores}. We omit the inputs when unambiguous; all HA-dLLM objectives below operate on this matrix.

\subsubsection{Forward process and training objectives}

\paragraph{Subset-shuffling corruption.}
We use subset shuffling, a shuffle-based corruption related to prior permutation diffusion processes~\citep{zhang2025symmetricdiffusers}. We sample $T\sim\operatorname{Unif}([K])$, select $T$ output positions uniformly without replacement, and uniformly permute the candidates at those positions while leaving all others unchanged. Let $g_T^{\mathrm{tr}}$ denote the resulting position permutation. The corrupted target is
\begin{equation}
  \pi_T=\pi_0\circ g_T^{\mathrm{tr}}.
  \label{eq:method-hard-forward}
\end{equation}

We explore four alternative training objectives for HA-dLLM: per-token cross-entropy, row-plus-column assignment cross-entropy, Sinkhorn marginal negative log-likelihood, and structured hinge.

\paragraph{Per-token cross-entropy.}
Per-token cross-entropy (token CE) applies cross-entropy independently to each row of $\mathbf{S}$:
\begin{equation}
  \mathcal{L}_{\mathrm{tok}}
  =-\frac{1}{K}
    \sum_{i=1}^K
    \log
    \frac{\exp S_{i,\pi_0(i)}}
         {\sum_{j=1}^{K}\exp S_{ij}}.
  \label{eq:method-token-ce}
\end{equation}

\paragraph{Row-plus-column assignment cross-entropy.}
Row-plus-column assignment cross-entropy (assignment CE) augments token CE with a column-wise term, encouraging each candidate to occupy exactly one output position. Let $\lambda_{\mathrm{col}}\geq0$ weight this term:
\begin{align}
  \mathcal{L}_{\mathrm{col}}
  &=-\frac{1}{K}
    \sum_{j=1}^K
    \log\frac{\exp S_{\pi_0^{-1}(j),j}}{\sum_{i=1}^K\exp S_{ij}},
    \label{eq:method-column-ce}\\
  \mathcal{L}_{\mathrm{assign}}
  &=\mathcal{L}_{\mathrm{tok}}
    +\lambda_{\mathrm{col}}\mathcal{L}_{\mathrm{col}}.
  \label{eq:method-assignment-ce}
\end{align}

\paragraph{Sinkhorn marginal negative log-likelihood.}
Using Sinkhorn normalization~\citep{sinkhorn1967matrices,mena2018gumbelsinkhorn} as introduced in Section~\ref{sec:prelim-assignment}, we construct the Sinkhorn marginal negative log-likelihood (Sinkhorn NLL) loss
\begin{equation}
  \mathbf{Q}=\operatorname{SK}_{\tau}(\mathbf{S}),
  \qquad
  \mathcal{L}_{\mathrm{sink}}
  =-\frac{1}{K}
    \sum_{i=1}^K\log Q_{i,\pi_0(i)}.
  \label{eq:method-hard-sinkhorn-loss}
\end{equation}
This objective assigns high marginal probability to the ground-truth candidate at every position while coupling rows and columns through Sinkhorn normalization.

\paragraph{Structured hinge.}
We adapt the margin-rescaled structured hinge~\citep{tsochantaridis2005structured} to global assignment scores. Let $\mathbf{P}_0$ be the clean permutation matrix defined in Equation~\eqref{eq:prelim-clean-permutation}, and define the per-position mismatch cost
\begin{equation}
  \Delta_{ij}=\gamma(1-P_{0,ij}),\qquad \gamma>0.
\end{equation}
The most violated assignment and the resulting hinge are
\begin{align}
  \widetilde{\mathbf{P}}
  &=\operatorname*{arg\,max}_{\mathbf{P}\in\mathcal{P}_K}
    \left\{\langle\mathbf{S},\mathbf{P}\rangle+\langle\boldsymbol{\Delta},\mathbf{P}\rangle\right\},
    \label{eq:method-loss-augmented-assignment}\\
  \mathcal{L}_{\mathrm{hinge}}
  &=\frac{1}{K}
    \left[
      \langle\mathbf{S},\widetilde{\mathbf{P}}-\mathbf{P}_0\rangle
      +\langle\boldsymbol{\Delta},\widetilde{\mathbf{P}}\rangle
    \right]_+.
  \label{eq:method-structured-hinge}
\end{align}
The Hungarian solution $\widetilde{\mathbf{P}}$ is detached when computing the subgradient, and $[z]_+=\max\{z,0\}$.

Algorithm~\ref{alg:ha-training} summarizes a single-example training step for the selected objective.

\begin{algorithm}[!t]
  \caption{HA-dLLM training step}
  \label{alg:ha-training}
  \small
  \begin{algorithmic}[1]
    \Require Training example $(c,\pi_0)$; model parameters $\theta$; selected HA-dLLM objective and its hyperparameters
    \State Sample $T\sim\operatorname{Unif}([K])$ and select $T$ positions uniformly without replacement.
    \State Sample $g_T^{\mathrm{tr}}$ by uniformly permuting the selected positions.
    \State $\pi_T\gets\pi_0\circ g_T^{\mathrm{tr}}$
    \State $\mathbf{S}\gets\mathbf{S}_\theta(c,\pi_T,T)$
    \State Compute the training loss using Equation~\eqref{eq:method-token-ce}, \eqref{eq:method-assignment-ce}, \eqref{eq:method-hard-sinkhorn-loss}, or \eqref{eq:method-structured-hinge}, according to the selected objective.
    \State Update $\theta$ with the gradient of the loss.
    \State \Return Updated parameters $\theta$.
  \end{algorithmic}
\end{algorithm}

\subsubsection{Predict--recorrupt sampling and decoding}
\paragraph{Initialization and schedule.}
We adapt the restore-and-recorrupt sampler of Cold Diffusion~\citep{bansal2023cold} to permutations and call the resulting rule \emph{Predict--recorrupt}. For $1\leq N\leq K$, choose an integer schedule
\begin{equation}
  K=d_N>d_{N-1}>\cdots>d_0=0,
  \label{eq:method-hard-schedule}
\end{equation}
and sample one position shuffle $g_{d_n}$ at each scheduled level from the corresponding subset-shuffling corruption. These shuffles are retained throughout inference, with $g_{d_0}=\mathrm{id}$ and $g_{d_N}$ uniform over $\mathfrak{S}_K$. We initialize $\pi^{(N)}=g_{d_N}$. For $n=N,N-1,\ldots,1$, $\pi^{(n)}$ denotes the input at level $d_n$, and one model evaluation produces the state $\pi^{(n-1)}$ at level $d_{n-1}$.

\paragraph{Hungarian decoding.}
At step $n$, the model produces $\mathbf{S}^{(n)}=\mathbf{S}_\theta(c,\pi^{(n)},d_n)$. The Hungarian algorithm~\citep{kuhn1955hungarian} gives
\begin{equation}
  \widehat{\mathbf{P}}_0^{(n)}
  =\operatorname*{arg\,max}_{\mathbf{P}\in\mathcal{P}_K}
    \langle \mathbf{S}^{(n)},\mathbf{P}\rangle.
  \label{eq:method-token-hungarian}
\end{equation}
\paragraph{Predict--recorrupt update.}
Let $\widehat\pi_0^{(n)}$ denote the corresponding clean ranking. Composing this prediction with the retained shuffle at the next corruption level gives
\begin{equation}
  \boxed{
  \pi^{(n-1)}
  =\widehat\pi_0^{(n)}\circ g_{d_{n-1}}.
  }
  \label{eq:method-hard-predict-recorrupt}
\end{equation}
For $n>1$, the sampler re-corrupts the clean prediction before the next evaluation. Because $g_{d_0}=\mathrm{id}$, the final output is $\widehat\pi=\pi^{(0)}=\widehat\pi_0^{(1)}$.

Algorithm~\ref{alg:ha-sampling} summarizes Predict--recorrupt sampling for HA-dLLM.

\begin{algorithm}[!t]
  \caption{HA-dLLM Predict--recorrupt sampling}
  \label{alg:ha-sampling}
  \small
  \begin{algorithmic}[1]
    \Require Context $c$; trained model $\mathbf{S}_\theta$; budget $1\leq N\leq K$
    \State Choose integer levels $K=d_N>d_{N-1}>\cdots>d_0=0$.
    \State Draw and retain position shuffles $g_{d_0},\ldots,g_{d_N}$ according to the chosen levels.
    \State $\pi^{(N)}\gets g_{d_N}$
    \For{$n=N,N-1,\ldots,1$}
      \State $\mathbf{S}^{(n)}\gets\mathbf{S}_\theta(c,\pi^{(n)},d_n)$
      \State Use Hungarian assignment to obtain $\widehat\pi_0^{(n)}\gets\operatorname*{arg\,max}_{\pi\in\mathfrak{S}_K}\sum_{i=1}^K S^{(n)}_{i,\pi(i)}$.
      \State $\pi^{(n-1)}\gets\widehat\pi_0^{(n)}\circ g_{d_{n-1}}$
    \EndFor
    \State \Return $\widehat\pi=\pi^{(0)}$.
\end{algorithmic}
\end{algorithm}

Appendix~\ref{sec:analysis-hard-sampling} proves that Predict--recorrupt preserves permutation validity for HA-dLLM and OS-dLLM (Proposition~\ref{prop:hard-state-validity}) and analyzes why direct composition-and-inversion analogues of Residual-correction are unsuitable for hard permutation states.

\subsection{Soft-Assignment Prediction (SA-dLLM)}
\label{sec:method-soft-assignment}

\subsubsection{Input and output representation}

SA-dLLM replaces discrete token states with continuous assignments in the Birkhoff polytope $\mathcal{B}_K$ defined in Section~\ref{sec:prelim-assignment}. Its input comprises the context, a corrupted assignment, and the normalized corruption level $t\in[0,1]$. Let $\mathbf{E}\in\mathbb{R}^{K\times d}$ contain the candidate-index embeddings as rows. A state $\mathbf{X}_t\in\mathcal{B}_K$ induces ranking-slot embeddings
\begin{equation}
  \mathbf{H}_t=\mathbf{X}_t\mathbf{E}.
  \label{eq:method-soft-embedding}
\end{equation}
Prompt-token embeddings are unchanged. For ranking slots, $\mathbf{H}_t$ replaces the standard embedding lookup at the first transformer layer, after which the timestep embedding is added.

The backbone produces $\mathbf{S}_\theta(c,\mathbf{X}_t,t)\in\mathbb{R}^{K\times K}$. Write this matrix as $\mathbf{S}$ and let $\kappa>0$ be a score bound. To control the scale inherited from the autoregressive vocabulary projection, we row-center and rescale the scores for each example:
\begin{align}
  S^\circ_{ij}
  &=S_{ij}-\frac{1}{K}\sum_{j'=1}^K S_{ij'},\\
  \lVert \mathbf{S}^\circ\rVert_{\max}
  &=\max_{i,j}|S^\circ_{ij}|,\\
  \widetilde{\mathbf{S}}
  &=\frac{\mathbf{S}^\circ}
    {\max\{1,\lVert \mathbf{S}^\circ\rVert_{\max}/\kappa\}}.
  \label{eq:method-soft-score-calibration}
\end{align}
Applying the Sinkhorn normalization from Section~\ref{sec:prelim-assignment} with the appropriate temperature produces the clean soft-assignment prediction
\begin{equation}
  \mathbf{Q}_\theta(c,\mathbf{X}_t,t)
  =\operatorname{SK}_{\tau}(\widetilde{\mathbf{S}})
  \in\mathcal{B}_K
  \label{eq:method-soft-prediction}
\end{equation}

\subsubsection{Forward process and anchored-rollout training}

For each example, we sample Gaussian logits independently of the target and context and apply Sinkhorn normalization to obtain a noise matrix:
\begin{equation}
  G_{ij}\overset{\mathrm{i.i.d.}}{\sim}\mathcal{N}(0,1),
  \qquad
  \mathbf{Z}=\operatorname{SK}_{\tau}(\mathbf{G}).
  \label{eq:method-soft-endpoint}
\end{equation}
Holding $\mathbf{Z}$ fixed, the forward path linearly interpolates between the clean permutation matrix and the noise matrix:
\begin{equation}
  \mathbf{X}_t^*=(1-t)\mathbf{P}_0+t\mathbf{Z},
  \qquad t\in[0,1].
  \label{eq:method-soft-forward}
\end{equation}

SA-dLLM training uses two model evaluations, at corruption levels $1$ and $t'<1$, to provide supervision on a model-generated state~\citep{ross2011dagger}. The first evaluation predicts a clean assignment from the noise endpoint:
\begin{equation}
  \mathbf{Q}^{\mathrm{end}}=\mathbf{Q}_\theta(c,\mathbf{Z},1).
\end{equation}
We then sample $t'$ and re-corrupt $\mathbf{Q}^{\mathrm{end}}$ with the same noise matrix:
\begin{equation}
  \widetilde{\mathbf{X}}_{t'}
  =(1-t')\operatorname{sg}(\mathbf{Q}^{\mathrm{end}})+t'\mathbf{Z},
  \label{eq:method-anchored-state}
\end{equation}
where $\operatorname{sg}$ denotes stop-gradient. The second evaluation predicts the same clean target from this re-corrupted assignment:
\begin{equation}
  \mathbf{Q}^{\mathrm{roll}}=\mathbf{Q}_\theta(c,\widetilde{\mathbf{X}}_{t'},t').
\end{equation}
The rollout state contains no direct contribution from $\mathbf{P}_0$. Because the transition begins at $\mathbf{X}_1=\mathbf{Z}$, Predict--recorrupt and Residual-correction yield the same $\widetilde{\mathbf{X}}_{t'}$ for this single step.

\paragraph{Anchored-rollout marginal loss.}
For a predicted assignment $\mathbf{Q}$ with positive entries, define the marginal loss over all output positions
\begin{equation}
  \ell_{\mathrm{SA}}(\mathbf{Q},\mathbf{P}_0)
  =-\frac{1}{K}\langle \mathbf{P}_0,\log \mathbf{Q}\rangle
  =-\frac{1}{K}\sum_{i=1}^K\log Q_{i,\pi_0(i)}.
  \label{eq:method-soft-marginal-loss}
\end{equation}
The anchored-rollout objective is
\begin{equation}
  \boxed{
  \mathcal{L}_{\mathrm{anchor}}
  =\frac{1}{2}\ell_{\mathrm{SA}}(\mathbf{Q}^{\mathrm{end}},\mathbf{P}_0)
   +\frac{1}{2}\ell_{\mathrm{SA}}(\mathbf{Q}^{\mathrm{roll}},\mathbf{P}_0).
  }
  \label{eq:method-anchored-loss}
\end{equation}
As with the hard Sinkhorn objective, this marginal surrogate scores the ground-truth candidate at each position rather than computing the exact negative log-likelihood of a joint distribution over $K!$ permutations.

Algorithm~\ref{alg:sa-training} summarizes the procedure. The set $\mathcal{N}_{\mathrm{train}}$ determines the grid from which $t'$ is sampled; each example still uses exactly two model evaluations, irrespective of $N_{\mathrm{tr}}$.

\begin{algorithm}[!t]
  \caption{SA-dLLM anchored-rollout training step}
  \label{alg:sa-training}
  \small
  \begin{algorithmic}[1]
    \Require Training example $(c,\mathbf{P}_0)$; model parameters $\theta$; budget set $\mathcal{N}_{\mathrm{train}}$; Sinkhorn temperature $\tau$
    \State Sample $G_{ij}\overset{\mathrm{i.i.d.}}{\sim}\mathcal{N}(0,1)$ and set $\mathbf{Z}\gets\operatorname{SK}_{\tau}(\mathbf{G})$.
    \State $\mathbf{Q}^{\mathrm{end}}\gets\mathbf{Q}_\theta(c,\mathbf{Z},1)$ using Equation~\eqref{eq:method-soft-prediction}.
    \State Sample $N_{\mathrm{tr}}\sim\operatorname{Unif}(\mathcal{N}_{\mathrm{train}})$ and $\ell\sim\operatorname{Unif}\{0,\ldots,N_{\mathrm{tr}}-1\}$.
    \State $t'\gets1-(\ell+1)/N_{\mathrm{tr}}$
    \State $\widetilde{\mathbf{X}}_{t'}\gets(1-t')\operatorname{sg}(\mathbf{Q}^{\mathrm{end}})+t'\mathbf{Z}$
    \State $\mathbf{Q}^{\mathrm{roll}}\gets\mathbf{Q}_\theta(c,\widetilde{\mathbf{X}}_{t'},t')$
    \State $\mathcal{L}_{\mathrm{anchor}}\gets\tfrac12\ell_{\mathrm{SA}}(\mathbf{Q}^{\mathrm{end}},\mathbf{P}_0)+\tfrac12\ell_{\mathrm{SA}}(\mathbf{Q}^{\mathrm{roll}},\mathbf{P}_0)$
    \State Update $\theta$ with the gradient of the single-example loss.
    \State \Return Updated parameters $\theta$.
  \end{algorithmic}
\end{algorithm}

\subsubsection{Predict--recorrupt and Residual-correction sampling}

\paragraph{Initialization and schedule.}
We adapt the restore-and-recorrupt and improved updates of Cold Diffusion~\citep{bansal2023cold}, calling them Predict--recorrupt and Residual-correction, respectively. Both use the same trained model, noise matrix, and terminal decoder. For an NFE budget $N\geq1$, they follow the linear schedule
\begin{equation}
  1=t_N>t_{N-1}>\cdots>t_0=0,
  \qquad t_n=\frac{n}{N},
  \label{eq:method-soft-schedule}
\end{equation}
sample one noise matrix $\mathbf{Z}$, initialize $\mathbf{X}_{t_N}=\mathbf{Z}$, and reuse $\mathbf{Z}$ throughout the trajectory. The model receives $(c,\mathbf{X}_{t_n},t_n)$, while the sampler additionally uses $\mathbf{Z}$ to construct the next state. At each step $n=N,N-1,\ldots,1$, the model predicts $\mathbf{Q}^{(n)}=\mathbf{Q}_\theta(c,\mathbf{X}_{t_n},t_n)$.

\paragraph{Predict--recorrupt update.}
Predict--recorrupt applies the forward corruption to the latest prediction at the next scheduled time:
\begin{equation}
  \boxed{
  \mathbf{X}_{t_{n-1}}
  =(1-t_{n-1})\mathbf{Q}^{(n)}+t_{n-1}\mathbf{Z}.
  }
  \label{eq:method-soft-predict-recorrupt}
\end{equation}
Its terminal assignment is the last clean prediction $\mathbf{A}=\mathbf{Q}^{(1)}$.

\paragraph{Residual-correction update.}
Residual-correction subtracts the current prediction corrupted at $t_n$ and adds the same prediction corrupted at $t_{n-1}$, using the retained noise matrix $\mathbf{Z}$. Under the affine corruption in Equation~\eqref{eq:method-soft-forward}, the update simplifies to
\begin{equation}
\boxed{
  \mathbf{X}_{t_{n-1}}
  =\mathbf{X}_{t_n}+(t_n-t_{n-1})(\mathbf{Q}^{(n)}-\mathbf{Z}).}
  \label{eq:method-soft-residual-correction}
\end{equation}
We implement this update with the equivalent accumulator
\begin{align}
  \mathbf{C}_N&=\mathbf{0}_{K\times K},\\
  \mathbf{C}_{n-1}&=\mathbf{C}_n+(t_n-t_{n-1})\mathbf{Q}^{(n)},\\
  \mathbf{X}_{t_{n-1}}&=t_{n-1}\mathbf{Z}+\mathbf{C}_{n-1}.
  \label{eq:method-soft-accumulator}
\end{align}
Its terminal assignment is the step-weighted ensemble
\begin{equation}
  \mathbf{A}=\mathbf{C}_0
  =\sum_{n=1}^{N}(t_n-t_{n-1})\mathbf{Q}^{(n)}.
  \label{eq:method-soft-ensemble}
\end{equation}

If $\mathbf{Z}$ and all $\mathbf{Q}^{(n)}$ are exactly doubly stochastic, both updates remain in $\mathcal{B}_K$ because every state is a convex combination of $\mathbf{Z}$ and the predictions. Finite Sinkhorn normalization satisfies these constraints only approximately. Appendix~\ref{sec:analysis-soft-sampling} proves the validity claims and derives the Residual-correction trajectory and terminal ensemble. Neither sampler uses an intermediate discrete projection.

\paragraph{Terminal Hungarian decoding.}
For either sampler, the terminal assignment scores are
\begin{equation}
  \mathbf{W}=\log\!\left(\max\{\mathbf{A},\varepsilon\}\right),
  \label{eq:method-soft-final-scores}
\end{equation}
where $\varepsilon>0$ provides a numerical floor before the elementwise logarithm. A single Hungarian projection then yields
\begin{equation}
  \widehat\pi
  =\operatorname*{arg\,max}_{\pi\in\mathfrak{S}_K}
    \sum_{i=1}^K W_{i,\pi(i)}.
  \label{eq:method-soft-final-projection}
\end{equation}
The soft-state updates neither clip nor renormalize intermediate states; the terminal Hungarian assignment is the only discrete projection.

Algorithms~\ref{alg:sa-predict-recorrupt-sampling} and~\ref{alg:sa-residual-correction-sampling} summarize Predict--recorrupt and Residual-correction sampling, respectively. Each uses one terminal Hungarian assignment.

\begin{algorithm}[!t]
  \caption{SA-dLLM Predict--recorrupt sampling}
  \label{alg:sa-predict-recorrupt-sampling}
  \small
  \begin{algorithmic}[1]
    \Require Context $c$; trained model $\mathbf{Q}_\theta$; budget $N\geq1$; Sinkhorn temperature $\tau>0$; numerical floor $\varepsilon>0$
    \State Set $t_n\gets n/N$ for $n=0,\ldots,N$.
    \State Sample $G_{ij}\overset{\mathrm{i.i.d.}}{\sim}\mathcal{N}(0,1)$ and set $\mathbf{Z}\gets\operatorname{SK}_{\tau}(\mathbf{G})$.
    \State $\mathbf{X}_{t_N}\gets\mathbf{Z}$
    \For{$n=N,N-1,\ldots,1$}
      \State $\mathbf{Q}^{(n)}\gets\mathbf{Q}_\theta(c,\mathbf{X}_{t_n},t_n)$
      \State $\mathbf{X}_{t_{n-1}}\gets(1-t_{n-1})\mathbf{Q}^{(n)}+t_{n-1}\mathbf{Z}$
    \EndFor
    \State $\mathbf{W}\gets\log\!\left(\max\{\mathbf{Q}^{(1)},\varepsilon\}\right)$
    \State Use Hungarian assignment to obtain $\widehat\pi\gets\operatorname*{arg\,max}_{\pi\in\mathfrak{S}_K}\sum_{i=1}^K W_{i,\pi(i)}$.
    \State \Return $\widehat\pi$.
  \end{algorithmic}
\end{algorithm}

\begin{algorithm}[!t]
  \caption{SA-dLLM Residual-correction sampling}
  \label{alg:sa-residual-correction-sampling}
  \small
  \begin{algorithmic}[1]
    \Require Context $c$; trained model $\mathbf{Q}_\theta$; budget $N\geq1$; Sinkhorn temperature $\tau>0$; numerical floor $\varepsilon>0$
    \State Set $t_n\gets n/N$ for $n=0,\ldots,N$.
    \State Sample $G_{ij}\overset{\mathrm{i.i.d.}}{\sim}\mathcal{N}(0,1)$ and set $\mathbf{Z}\gets\operatorname{SK}_{\tau}(\mathbf{G})$.
    \State $\mathbf{X}_{t_N}\gets\mathbf{Z}$; $\mathbf{C}_N\gets\mathbf{0}_{K\times K}$
    \For{$n=N,N-1,\ldots,1$}
      \State $\mathbf{Q}^{(n)}\gets\mathbf{Q}_\theta(c,\mathbf{X}_{t_n},t_n)$
      \State $\mathbf{C}_{n-1}\gets\mathbf{C}_n+(t_n-t_{n-1})\mathbf{Q}^{(n)}$
      \State $\mathbf{X}_{t_{n-1}}\gets t_{n-1}\mathbf{Z}+\mathbf{C}_{n-1}$
    \EndFor
    \State $\mathbf{W}\gets\log\!\left(\max\{\mathbf{C}_0,\varepsilon\}\right)$
    \State Use Hungarian assignment to obtain $\widehat\pi\gets\operatorname*{arg\,max}_{\pi\in\mathfrak{S}_K}\sum_{i=1}^K W_{i,\pi(i)}$.
    \State \Return $\widehat\pi$.
  \end{algorithmic}
\end{algorithm}

\subsection{Direct Ordering-Score Prediction (OS-dLLM)}
\label{sec:method-permutation-head}

\subsubsection{Input and output representation}

OS-dLLM uses the same prompt, hard permutation state, slot-attention pattern, and Fourier-feature timestep embedding as HA-dLLM. Rather than scoring assignments or enumerating $K!$ permutations, it maps the hidden state of each current source slot $s$ to a scalar ordering score
\begin{equation}
  q_s=\mathbf{a}^\top \mathbf{h}_s+b,
  \qquad s\in[K].
  \label{eq:method-ordering-key}
\end{equation}
Here $\mathbf{a}\in\mathbb{R}^{d}$ and $b\in\mathbb{R}$ are learned parameters. Sorting the $K$ scores determines the destination order. We denote the source-slot score vector by $\mathbf{q}_\theta(c,\pi_T,T)\in\mathbb{R}^K$.

\subsubsection{Forward process and training loss}

Training uses the subset-shuffling corruption in Equation~\eqref{eq:method-hard-forward}. Relative to the current ranking $\pi_T$, define the target source order
\begin{equation}
  \sigma_T^*=\pi_T^{-1}\circ\pi_0,
  \qquad
  \pi_T\!\left(\sigma_T^*(i)\right)=\pi_0(i).
  \label{eq:method-target-source-order}
\end{equation}
Thus $\sigma_T^*(i)$ is the current source slot containing the candidate assigned to destination rank $i$.

\paragraph{Listwise maximum likelihood.}
We train all $K$ positions with the negative log-likelihood of the Plackett--Luce model~\citep{plackett1975permutations}, using the ListMLE objective~\citep{xia2008listmle}:
\begin{equation}
  \mathcal{L}_{\mathrm{perm}}
  =\frac{1}{K}\sum_{i=1}^K
    \left[
      \log\sum_{r=i}^K\exp q_{\sigma_T^*(r)}
      -q_{\sigma_T^*(i)}
    \right].
  \label{eq:method-listmle}
\end{equation}

Algorithm~\ref{alg:os-training} summarizes the ListMLE training step for one example.

\begin{algorithm}[!t]
  \caption{OS-dLLM training step}
  \label{alg:os-training}
  \small
  \begin{algorithmic}[1]
    \Require Training example $(c,\pi_0)$; model parameters $\theta$
    \State Sample $T\sim\operatorname{Unif}([K])$ and select $T$ positions uniformly without replacement.
    \State Sample $g_T^{\mathrm{tr}}$ by uniformly permuting the selected positions.
    \If{$g_T^{\mathrm{tr}}$ is the identity}
      \State Replace $g_T^{\mathrm{tr}}$ by a random transposition, following Section~\ref{sec:experiments-protocol}.
    \EndIf
    \State $\pi_T\gets\pi_0\circ g_T^{\mathrm{tr}}$
    \State Compute source-slot scores $\mathbf{q}\gets\mathbf{q}_\theta(c,\pi_T,T)$.
    \State $\sigma_T^*\gets\pi_T^{-1}\circ\pi_0$
    \State Compute $\mathcal{L}_{\mathrm{perm}}$ over all $K$ positions using Equation~\eqref{eq:method-listmle}.
    \State Update $\theta$ with the gradient of the single-example loss.
    \State \Return Updated parameters $\theta$.
  \end{algorithmic}
\end{algorithm}

\subsubsection{Predict--recorrupt sampling and decoding}

\paragraph{Sorting-based decoding.}
For $n=N,N-1,\ldots,1$, let $\mathbf{q}^{(n)}\in\mathbb{R}^K$ contain the source-slot scores. A stable descending sort yields the source order
\begin{equation}
  \widehat\sigma^{(n)}
  =\operatorname{argsort}(\mathbf{q}^{(n)};\text{descending}),
\end{equation}
which is composed with the current ranking to predict the clean state:
\begin{equation}
  \widehat\pi_0^{(n)}
  =\pi^{(n)}\circ\widehat\sigma^{(n)}.
  \label{eq:method-permutation-head-decode}
\end{equation}
The decoder therefore uses sorting rather than Hungarian assignment.

\paragraph{Predict--recorrupt update.}
Multi-step inference uses the same fixed trajectory and Predict--recorrupt update as Equation~\eqref{eq:method-hard-predict-recorrupt}:
\begin{equation}
  \pi^{(n-1)}
  =\widehat\pi_0^{(n)}\circ g_{d_{n-1}}.
  \label{eq:method-permutation-predict-recorrupt}
\end{equation}
Algorithm~\ref{alg:os-sampling} summarizes sorting-based Predict--recorrupt sampling.

\begin{algorithm}[!t]
  \caption{OS-dLLM Predict--recorrupt sampling}
  \label{alg:os-sampling}
  \small
  \begin{algorithmic}[1]
    \Require Context $c$; trained model $\mathbf{q}_\theta$; budget $1\leq N\leq K$
    \State Choose integer levels $K=d_N>d_{N-1}>\cdots>d_0=0$.
    \State Draw and retain position shuffles $g_{d_0},\ldots,g_{d_N}$ as in Algorithm~\ref{alg:ha-sampling}.
    \State $\pi^{(N)}\gets g_{d_N}$
    \For{$n=N,N-1,\ldots,1$}
      \State $\mathbf{q}^{(n)}\gets\mathbf{q}_\theta(c,\pi^{(n)},d_n)$
      \State $\widehat\sigma^{(n)}\gets\operatorname{argsort}(\mathbf{q}^{(n)};\text{descending})$ using a stable sort.
      \State $\widehat\pi_0^{(n)}\gets\pi^{(n)}\circ\widehat\sigma^{(n)}$
      \State $\pi^{(n-1)}\gets\widehat\pi_0^{(n)}\circ g_{d_{n-1}}$
    \EndFor
    \State \Return $\widehat\pi=\pi^{(0)}$.
  \end{algorithmic}
\end{algorithm}

%% file: files/06_experiments.tex

\section{Experiments}
\label{sec:experiments}

We evaluate one-evaluation ranking quality, projection-free validity, and sensitivity to the NFE budget. We also compare the training objectives and sampling procedures within each model family.

\subsection{Experimental Setup}
\label{sec:experiments-setup}

\subsubsection{Datasets and splits}

We use the Beauty, Sports and Outdoors, and Toys and Games categories of the Amazon-2014 review dataset~\citep{mcauley2015image}\footnote{Evaluation on the internal system is ongoing.}. We apply standard 5-core filtering and a leave-one-out protocol, yielding one query per retained user in each split. Table~\ref{tab:dataset-statistics} summarizes the processed data.

\begin{table}[!t]
  \centering
  \small
  \caption{Statistics of the three processed Amazon review datasets.}
  \label{tab:dataset-statistics}
  \begin{tabular}{lrrrr}
    \toprule
    Dataset & \# Users & \# Items & Mean Len. & Median Len. \\
    \midrule
    Beauty              & $22{,}363$ & $12{,}101$ & 8.87 & 6 \\
    Sports and Outdoors & $35{,}598$ & $18{,}357$ & 8.32 & 6 \\
    Toys and Games      & $19{,}412$ & $11{,}924$ & 8.63 & 6 \\
    \bottomrule
  \end{tabular}
\end{table}

\subsubsection{Candidate construction and target relabeling}
\label{sec:experiments-candidates}


Each reranking example contains user-history text and textual descriptions of $K$ candidates. The target is a sequence of candidate indices ordered from most to least relevant.

For each query, a pointwise cross-encoder reranker~\citep{xu2026case} scores the item catalog. For $K\in\{10,20\}$, we retain the held-out positive and the $K-1$ highest-scoring nonpositive items as hard negatives. The positive is the only item with an observed relevance label. We place it first in the clean order and sort the negatives by decreasing cross-encoder score; consequently, the negative order is a pseudo-label rather than an independently annotated relevance judgment.


Before presentation, we randomly permute the candidate list and relabel the target with the inverse permutation. Formally, let $\bar x_j$ and $\bar\pi_0$ denote the canonical candidate descriptions and target order, where $\bar\pi_0$ is the identity under the construction above. We sample $\eta\in\mathfrak{S}_K$, present $x_j=\bar x_{\eta(j)}$, and set $\pi_0=\eta^{-1}\circ\bar\pi_0$. This presentation shuffle is part of dataset construction and is distinct from diffusion corruption.

\subsubsection{Prompt and output representation}

The prompt contains a system instruction, the user's purchase history, and $K$ enumerated candidate descriptions. Candidate $j$ is prefixed by \texttt{<idx\_j>}. The final instruction requests every index exactly once, ordered from most to least relevant and enclosed by \texttt{<rank>} and \texttt{</rank>}.


\subsection{Metrics}
\label{sec:experiments-metrics}

Let $r\in\{1,\ldots,K\}$ denote the one-based rank of the held-out positive in a valid prediction. Invalid outputs receive zero for all relevance metrics. We average each metric over test queries.

\paragraph{Recall.}
Recall@$k$ indicates whether the positive appears among the first $k$ positions:
\begin{equation}
  \operatorname{R@}k=\mathbf{1}[r\leq k].
  \label{eq:experiment-relevance-metrics}
\end{equation}
We report Recall@1 and Recall@5.

\paragraph{Mean reciprocal rank.}
Mean reciprocal rank (MRR) is
\begin{equation}
  \operatorname{MRR}=\frac{1}{r}.
\end{equation}

\paragraph{Normalized discounted cumulative gain.}
With one relevant item, normalized discounted cumulative gain at cutoff $k$ is
\begin{equation}
  \operatorname{NDCG@}k
  =\frac{\mathbf{1}[r\leq k]}{\log_2(r+1)}.
\end{equation}
We report NDCG@1 and NDCG@5. Because every example has exactly one relevant item, NDCG@1 and Recall@1 are numerically identical; duplicate NDCG@1 values are omitted from the main tables.

\paragraph{Position accuracy.}
Position accuracy is the fraction of positions whose predicted candidate matches the clean target at the same position. It measures recovery of the full target permutation rather than only the positive item's rank.

\paragraph{Pairwise accuracy.}
Pairwise accuracy is the fraction of candidate pairs that appear in the same relative order as in the clean target. A pair receives credit only if both candidates occur exactly once, thereby penalizing omissions and duplicates.

\paragraph{Valid rate.}
Valid rate is the fraction of outputs containing every presented candidate index exactly once. We apply no post hoc projection during evaluation; position and pairwise accuracy are computed from the raw output.


\paragraph{Number of function evaluations.}
The number of function evaluations (NFE) counts language-model forward evaluations used to produce a ranking. AR requires one sequential evaluation per candidate token, whereas each diffusion evaluation scores all $K$ ranking positions in parallel.

\subsection{Baselines}
\label{sec:experiments-models}

All learned models use the same pretrained Qwen3-0.6B backbone\footnote{Model checkpoint: \url{https://huggingface.co/Qwen/Qwen3-0.6B}. Official overview: \url{https://qwenlm.github.io/blog/qwen3/}.}, candidate sets, prompt format, candidate-index vocabulary, and presentation shuffle. We include the following baselines:
\begin{itemize}
  \item \textbf{Random valid permutation.} A uniformly random candidate order provides a nonlearned reference with guaranteed validity.
  \item \textbf{AR reranker.} The model generates $K$ candidate-index tokens with permutation-constrained greedy decoding.
  \item \textbf{Masked diffusion language model.} The mdLLM predicts masked slots in parallel, and its constrained decoder commits a one-to-one assignment over the remaining positions and candidates at each step.
\end{itemize}

The evaluated configurations follow Table~\ref{tab:method-design-space}. HA-dLLM uses token CE, row-plus-column assignment CE, Sinkhorn marginal NLL, or structured hinge training, followed by Predict--recorrupt sampling. SA-dLLM uses Predict--recorrupt or Residual-correction sampling, and OS-dLLM uses ListMLE with Predict--recorrupt sampling. HA-dLLM applies Hungarian decoding at every denoising step, SA-dLLM applies it only at termination, and OS-dLLM uses sorting.




Complete optimization, checkpoint-selection, and evaluation settings are provided in Appendix~\ref{sec:implementation-details}.

\subsection{Parallel Reranking with One Function Evaluation}
\label{sec:experiments-main}

\begin{table*}[!t]
  \centering
  \small
  \caption{Full-test Recall@1. AR uses $K$ sequential function evaluations; all other learned models use NFE$=1$. ``H'' denotes Hungarian assignment and ``sort'' denotes the ordering-score decoder. AR values marked $\dagger$ are from earlier full-test controls; all other values are from the metric-complete evaluation. Boldface indicates the highest pre-rounding value in each column.}
  \label{tab:main-one-call}
  \begin{tabular}{lcrrrrrr}
    \toprule
    & & \multicolumn{3}{c}{$K=10$} & \multicolumn{3}{c}{$K=20$} \\
    \cmidrule(lr){3-5}\cmidrule(lr){6-8}
    Model / decoder & NFE & Beauty & Sports & Toys & Beauty & Sports & Toys \\
    \midrule
    Random valid permutation & 0 & .100 & .100 & .100 & .050 & .050 & .050 \\
    AR$^{\dagger}$ & $K$ & .863 & .892 & \textbf{.852} & \textbf{.828} & .875 & \textbf{.792} \\
    \midrule
    mdLLM + constrained/H & 1 & .883 & .917 & .840 & .823 & \textbf{.875} & .787 \\
    HA-dLLM, token CE + H & 1 & .879 & .913 & .845 & .817 & .872 & .782 \\
    HA-dLLM, assignment CE + H & 1 & .880 & .918 & .847 & .820 & .873 & .790 \\
    HA-dLLM, Sinkhorn NLL + H & 1 & .881 & \textbf{.918} & .847 & .817 & .871 & .786 \\
    HA-dLLM, structured hinge + H & 1 & .855 & .885 & .848 & .813 & .854 & .771 \\
    SA-dLLM anchored + final H & 1 & \textbf{.889} & .914 & .851 & .823 & .869 & .788 \\
    OS-dLLM + sort & 1 & .840 & .887 & .814 & .782 & .842 & .741 \\
    \bottomrule
  \end{tabular}
\end{table*}

Table~\ref{tab:main-one-call} compares NFE$=1$ diffusion prediction with $K$-step AR decoding. The strongest diffusion configuration exceeds the reported AR control on Beauty and Sports at $K=10$, matches it on Sports at $K=20$, and differs by at most $0.5$ Recall@1 percentage points in the remaining settings. No diffusion family dominates all six settings: SA-dLLM performs best on Beauty at $K=10$ and matches constrained mdLLM on Beauty at $K=20$, whereas HA-dLLM variants lead the proposed families on Sports at both candidate-set sizes and on Toys at $K=20$. Because the AR entries come from earlier full-test controls, these cross-family comparisons should be interpreted as provisional until AR is regenerated with the metric-complete evaluator.


\subsection{Constraint Satisfaction Without Assignment Projection}
\label{sec:experiments-validity}

Table~\ref{tab:projection-free} compares discrete rankings produced without assignment projection. At NFE$=1$, raw mdLLM has near-zero validity because independent token predictions repeat or omit candidates, whereas OS-dLLM is valid by construction. At NFE$=K$, raw mdLLM approaches unit validity and achieves higher Recall@1 than OS-dLLM in all six settings. Sorting therefore guarantees validity at every budget, whereas raw mdLLM relies on iterative commitment to recover valid permutations.

\begin{table*}[!t]
  \centering
  \small
  \caption{Projection-free Recall@1 and valid rate at NFE$=1$ and NFE$=K$. Each cell reports $N=1\rightarrow N=K$. Values are rounded to three decimals; 1.000 does not necessarily indicate exact validity of one.}
  \label{tab:projection-free}
  \begin{tabular}{lcll}
    \toprule
    Dataset & $K$ & Raw mdLLM & OS-dLLM \\
    \midrule
    Beauty & 10 & .013/.013 $\rightarrow$ .883/1.000 & .840/1.000 $\rightarrow$ .854/1.000 \\
    Beauty & 20 & .000/.000 $\rightarrow$ .823/1.000 & .782/1.000 $\rightarrow$ .787/1.000 \\
    Sports & 10 & .039/.040 $\rightarrow$ .916/1.000 & .887/1.000 $\rightarrow$ .898/1.000 \\
    Sports & 20 & .000/.000 $\rightarrow$ .874/1.000 & .842/1.000 $\rightarrow$ .846/1.000 \\
    Toys & 10 & .001/.001 $\rightarrow$ .842/1.000 & .814/1.000 $\rightarrow$ .820/1.000 \\
    Toys & 20 & .000/.000 $\rightarrow$ .789/1.000 & .741/1.000 $\rightarrow$ .739/1.000 \\
    \bottomrule
  \end{tabular}
\end{table*}


\subsection{Effect of Additional Function Evaluations}
\label{sec:experiments-steps}

Table~\ref{tab:beauty-k20-steps} reports Recall@1 across NFE budgets on Beauty with $K=20$. Constrained mdLLM, Sinkhorn-trained HA-dLLM, and SA-dLLM change little; token CE and assignment CE degrade, whereas OS-dLLM improves by $0.5$ percentage points. The effect of additional evaluations is therefore parameterization dependent.

\begin{table*}[!t]
  \centering
  \small
  \caption{Recall@1 versus NFE on Beauty with $K=20$. HA-dLLM and OS-dLLM use Predict--recorrupt sampling; SA-dLLM uses the selected anchored-rollout checkpoint.}
  \label{tab:beauty-k20-steps}
  \begin{tabular}{lrrrrr}
    \toprule
    Model / decoder & $N=1$ & $N=2$ & $N=5$ & $N=10$ & $N=20$ \\
    \midrule
    mdLLM + constrained/H & .823 & .822 & .822 & .823 & .823 \\
    HA-dLLM, token CE + H & .817 & .802 & .793 & .784 & .796 \\
    HA-dLLM, assignment CE + H & .820 & .811 & .814 & .813 & .811 \\
    HA-dLLM, Sinkhorn NLL + H & .817 & .817 & .819 & .817 & .818 \\
    HA-dLLM, structured hinge + H & .813 & .810 & .810 & .809 & .810 \\
    SA-dLLM, Predict--recorrupt & .823 & .822 & .822 & .822 & .822 \\
    SA-dLLM, Residual-correction & .823 & .822 & .823 & .823 & .823 \\
    OS-dLLM + sort & .782 & .781 & .783 & .784 & .787 \\
    \bottomrule
  \end{tabular}
\end{table*}

Table~\ref{tab:beauty-k20-deltas} shows that stable Recall@1 can conceal changes in the full order. OS-dLLM gains $0.497$ Recall@1 percentage points while losing $0.788$ position-accuracy and $0.860$ pairwise-accuracy points. Sinkhorn-trained HA-dLLM and SA-dLLM remain stable on both relevance and full-order metrics.

\begin{table*}[!t]
  \centering
  \small
  \caption{Metric changes from NFE$=1$ to NFE$=20$ on Beauty with $K=20$, in percentage points. Positive values indicate improvement.}
  \label{tab:beauty-k20-deltas}
  \begin{tabular}{lrrrrrr}
    \toprule
    Model / decoder & $\Delta$R@1 & $\Delta$R@5 & $\Delta$MRR & $\Delta$NDCG@5
      & $\Delta$Position & $\Delta$Pairwise \\
    \midrule
    mdLLM + constrained/H & $-.013$ & $+.004$ & $+.000$ & $-.006$ & $-.078$ & $-.128$ \\
    HA-dLLM, token CE + H & $-2.070$ & $-2.209$ & $-1.959$ & $-2.123$ & $-2.328$ & $-.965$ \\
    HA-dLLM, assignment CE + H & $-.970$ & $-.974$ & $-.907$ & $-.964$ & $-2.194$ & $-.570$ \\
    HA-dLLM, Sinkhorn NLL + H & $+.045$ & $+.080$ & $+.051$ & $+.058$ & $+.074$ & $+.005$ \\
    HA-dLLM, structured hinge + H & $-.295$ & $-.165$ & $-.241$ & $-.231$ & $-1.591$ & $-1.314$ \\
    SA-dLLM, Predict--recorrupt & $-.026$ & $-.067$ & $-.039$ & $-.042$ & $+.014$ & $+.077$ \\
    SA-dLLM, Residual-correction & $+.009$ & $-.005$ & $+.002$ & $+.003$ & $+.008$ & $+.032$ \\
    OS-dLLM + sort & $+.497$ & $-.336$ & $+.172$ & $+.040$ & $-.788$ & $-.860$ \\
    \bottomrule
  \end{tabular}
\end{table*}

Across the six evaluated settings, Sinkhorn NLL is the most consistently stable HA-dLLM objective as NFE increases. OS-dLLM improves from NFE$=1$ to NFE$=K$ in five settings but remains below the strongest HA-dLLM in each setting.

%% file: files/07_conclusion.tex
\section{Conclusion}
\label{sec:conclusion}

We studied conditional permutation generation with language-model denoisers through three parameterizations: hard-assignment prediction, soft-assignment prediction, and direct ordering-score prediction. The corresponding model families impose permutation constraints through per-step assignment, terminal assignment, and sorting, respectively, and require distinct corruption processes, objectives, and samplers.

The experiments demonstrate that useful rankings can be produced at NFE$=1$ while exposing tradeoffs among ranking quality, validity, and iterative stability. Anchored SA-dLLM remains stable across NFE budgets, whereas OS-dLLM guarantees validity without Hungarian assignment but trails the strongest constrained decoders. Additional function evaluations do not consistently improve ranking quality. These findings support treating corruption, prediction parameterization, objective, and sampler as coupled design choices. Evaluation on larger candidate sets, multiple training seeds, independently judged full rankings, and controlled end-to-end latency measurements remains important future work.

%% file: files/08_appendix.tex
\section{Implementation Details}
\label{sec:implementation-details}

This section specifies the input construction, optimization, checkpoint-selection, and evaluation settings used in the experiments.

\subsection{Input Serialization and Sequence Lengths}

We extend the pretrained vocabulary $\mathcal{V}_{\mathrm{raw}}$ with the candidate-index tokens $\mathcal{V}_{\mathrm{idx}}$, the mdLLM mask token \texttt{<rank\_mask>}, and the delimiters \texttt{<rank>} and \texttt{</rank>}. The serialized AR target contains $K+2$ tokens: an opening delimiter, $K$ candidate indices, and a closing delimiter marking the end of the ranking. Equation~\eqref{eq:prelim-ar-loss} includes only the candidate-token component. Input states and prediction parameterizations follow Sections~\ref{sec:preliminaries} and~\ref{sec:method}. The maximum sequence length is 2,048 for $K=10$ and 4,096 for $K=20$. When truncation is required, we left-truncate only the prompt prefix and preserve the complete ranking region.

\subsection{Training Configuration}
\label{sec:experiments-protocol}

\subsubsection{Initialization and shared optimization}

Each diffusion model is initialized from the AR checkpoint trained for the same dataset and candidate budget. The Fourier-feature MLP and scalar ordering-score projection, when present, are initialized separately; assignment-based models reuse the pretrained vocabulary projection. We train with AdamW, weight decay $0.01$, 5\% linear warmup followed by cosine decay, gradient-norm clipping at 1.0, a bfloat16 backbone, gradient checkpointing, and random seed 42. Numerically sensitive assignment mixing, Sinkhorn normalization, and trajectory arithmetic use FP32. The effective batch size is 128, and training runs for at most ten epochs. Diffusion models use a learning rate of $10^{-4}$; AR learning rates are selected from $\{2\times10^{-5},10^{-4}\}$ on validation data. All checkpoint selection uses validation data only.

\subsubsection{Hard-state models}

Both hard-state families use explicit timestep conditioning, sample $T$ uniformly from $[K]$, and omit self-conditioning. The Fourier-feature MLP receives $T/K$ during training and $d_n/K$ during inference. If a sampled training shuffle leaves the permutation unchanged, including when $T=1$, we replace it with a random transposition. Identity corruptions are therefore excluded during training, and the $T=K$ training distribution is not uniform over all permutations, whereas inference begins from a uniform permutation. HA-dLLM objectives supervise all $K$ positions, including positions unchanged by corruption; OS-dLLM applies ListMLE to the full order. Assignment CE uses $\lambda_{\mathrm{col}}=1$, hard Sinkhorn NLL uses $\tau=0.1$ with 20 Sinkhorn iterations, and structured hinge uses $\gamma=1$.

\subsubsection{Soft-state model}

SA-dLLM uses noise-endpoint temperature $\tau_Z=0.7$, prediction temperature $\tau_Q=1.0$, score bound $\kappa=4.0$, and Sinkhorn tolerance $10^{-5}$. Sinkhorn normalization is capped at 256 iterations, with one optional 16-iteration refinement for near-threshold cases. The anchored objective uses two supervised backbone evaluations per example and samples $N_{\mathrm{tr}}$ from $\mathcal{N}_{\mathrm{train}}=\{2,5,10\}$ for $K=10$ and from $\mathcal{N}_{\mathrm{train}}=\{2,5,10,20\}$ for $K=20$.







\subsection{Evaluation Configuration}

For $K=10$, we evaluate $N\in\{1,2,5,10\}$; for $K=20$, we evaluate $N\in\{1,2,5,10,20\}$. The NFE$=1$ SA-dLLM prediction is identical for Predict--recorrupt and Residual-correction and is counted once. For each SA-dLLM example, we reuse the same deterministic noise endpoint across budgets and samplers. For HA-dLLM and OS-dLLM, we report only Predict--recorrupt; preliminary hard-state Residual-correction runs are excluded.

All full-test evaluations use eight parallel workers. The per-worker batch size is 32 for hard-state models and SA-dLLM with $K=10$, and 16 for SA-dLLM with $K=20$. We save both predicted and target permutations.

%% file: files/09_hard_sampling_analysis.tex
\subsection{Analysis of Hard-State Sampling}
\label{sec:analysis-hard-sampling}

We analyze the Predict--recorrupt updates for HA-dLLM and OS-dLLM, defined in Equations~\eqref{eq:method-hard-predict-recorrupt} and~\eqref{eq:method-permutation-predict-recorrupt} and summarized in Algorithms~\ref{alg:ha-sampling} and~\ref{alg:os-sampling}.

\subsubsection{Validity of Hard-State Predict--Recorrupt Sampling}
Under the schedule in Equation~\eqref{eq:method-hard-schedule}, one shuffle $g_{d_n}$ is sampled for each corruption level and retained throughout the denoising trajectory. For this fixed collection, define $D_d(\pi)=\pi\circ g_d$. Predict--recorrupt applies $D_{d_{n-1}}$ to each decoded clean prediction. Validity follows from closure of permutations under composition.

\begin{proposition}[Hard-state validity of Predict--recorrupt sampling]
\label{prop:hard-state-validity}
Suppose $\pi^{(N)}\in\mathfrak{S}_K$, every decoded clean prediction $\widehat\pi_0^{(n)}\in\mathfrak{S}_K$, and every retained shuffle $g_{d_n}\in\mathfrak{S}_K$, with $g_{d_0}=\mathrm{id}$. For any schedule $K=d_N>\cdots>d_0=0$ and every $n=N,\ldots,1$, Predict--recorrupt satisfies
\begin{equation}
  \pi^{(n-1)}
  =D_{d_{n-1}}(\widehat\pi_0^{(n)})
  \in\mathfrak{S}_K,
  \qquad
  \pi^{(0)}=\widehat\pi_0^{(1)}.
  \label{eq:analysis-hard-validity}
\end{equation}
\end{proposition}

\begin{proof}
Both $\widehat\pi_0^{(n)}$ and $g_{d_{n-1}}$ are bijections on $[K]$; their composition is therefore a bijection. Induction from the valid initial state establishes validity at every step. At the final step, $g_{d_0}=\mathrm{id}$ returns the last decoded clean prediction. Hungarian assignment in Equation~\eqref{eq:method-token-hungarian} guarantees a valid HA-dLLM prediction. For OS-dLLM, the stable sort in Equation~\eqref{eq:method-permutation-head-decode} produces a permutation of source slots, whose composition with the current permutation is also valid. The result is independent of prediction accuracy.
\end{proof}

The update is also consistent with a fixed clean prediction. If $\widehat\pi_0^{(n)}=\pi_\star$ at every evaluation, then each updated state equals $D_{d_{n-1}}(\pi_\star)$ and the final output is $\pi_\star$. This claim concerns states produced after the first evaluation. The sampled initial state is $g_{d_N}$, whereas the corresponding forward-corrupted state would be $\pi_\star\circ g_{d_N}$; these states need not coincide for the same shuffle. When $g_{d_N}$ is uniform and $\pi_\star$ is fixed, however, both have the uniform marginal distribution over $\mathfrak{S}_K$.

\input{files/09_permutation_residual_analysis.tex}

%% file: files/09_permutation_residual_analysis.tex
\subsubsection{Why Additive Residual-Correction Does Not Directly Extend to Permutations}

\paragraph{Original residual update.}
For a state $\mathbf{x}^{(n)}$, predicted clean state $\widehat{\mathbf{x}}_0^{(n)}$, and fixed corruption trajectory $\mathcal{D}_d$, Residual-correction uses
\begin{equation}
  \mathbf{x}^{(n-1)}
  =\mathbf{x}^{(n)}
   -\mathcal{D}_{d_n}(\widehat{\mathbf{x}}_0^{(n)})
   +\mathcal{D}_{d_{n-1}}(\widehat{\mathbf{x}}_0^{(n)}).
  \label{eq:group-residual-original-additive}
\end{equation}
The update preserves the difference between the current state and corrupted prediction while changing the corruption level.

\paragraph{Permutation-space analogue.}
For $D_d(\pi)=\pi\circ g_d$, replacing subtraction and addition with inversion and composition gives
\begin{align}
  \pi^{(n-1)}
  &=\pi^{(n)}
    \circ\bigl(D_{d_n}(\widehat\pi_0^{(n)})\bigr)^{-1}
    \circ D_{d_{n-1}}(\widehat\pi_0^{(n)})
    \label{eq:group-residual-right-update}\\
  &=\pi^{(n)}\circ(g_{d_n})^{-1}
    \circ\underbrace{(\widehat\pi_0^{(n)})^{-1}
    \circ\widehat\pi_0^{(n)}}_{\mathrm{id}}
    \circ g_{d_{n-1}}\notag\\
  &=\pi^{(n)}\circ(g_{d_n})^{-1}\circ g_{d_{n-1}}.
    \label{eq:group-residual-cancellation}
\end{align}
This update preserves permutation validity, but the prediction cancels within each step regardless of how predictions vary across the trajectory.

\paragraph{Dependence on initialization.}
For $K=d_N>\cdots>d_0=0$ and $g_{d_0}=\mathrm{id}$, successive shuffle
factors cancel across steps, giving
\begin{equation}
  \pi^{(0)}=\pi^{(N)}\circ(g_{d_N})^{-1}
  =\begin{cases}
    \pi_\star, & \pi^{(N)}=\pi_\star\circ g_{d_N},\\
    \mathrm{id}, & \pi^{(N)}=g_{d_N}.
  \end{cases}
  \label{eq:group-residual-general-terminal}
\end{equation}
Thus, the update exactly reconstructs a clean permutation from an input paired with its retained shuffle. Under the generation initialization $\pi^{(N)}=g_{d_N}$, however, it always returns the original presentation order independently of the model. For uniform $g_{d_N}$, the two initial states have identical marginals but different couplings with the retained shuffle unless $\pi_\star=\mathrm{id}$; the marginals alone therefore do not determine the output.

\paragraph{Connection to the TACoS analysis.}
Section~3.3 of Cold Diffusion analyzes Transformation Agnostic Cold Sampling (TACoS) under the approximation $\mathcal{D}_d(\mathbf{x})\approx\mathbf{x}+d\mathbf{e}$, where $\mathbf{e}$ is constant and independent of $\mathbf{x}$~\citep{bansal2023cold}. Treating this approximation as exact yields
\begin{align}
  \mathbf{x}^{(n-1)}
  &=\mathbf{x}^{(n)}
    -(\widehat{\mathbf{x}}_0^{(n)}+d_n\mathbf{e})
    +(\widehat{\mathbf{x}}_0^{(n)}+d_{n-1}\mathbf{e})\notag\\
  &=\mathbf{x}^{(n)}-(d_n-d_{n-1})\mathbf{e}.
  \label{eq:group-residual-tacos-translation}
\end{align}
Prediction errors cancel, so initialization at $\mathbf{x}^{(N)}=\mathbf{x}_\star+d_N\mathbf{e}$ exactly recovers $\mathbf{x}_\star$. This reconstruction property is consistent with the paired-permutation case above. Cancellation alone therefore does not imply failure; the degeneracy arises from the combination of the update and the generation initialization. For the soft corruption $D_t^{\mathbf Z}(\mathbf A)=\mathbf A+t(\mathbf Z-\mathbf A)$, the displacement depends on the clean state, so the residual update retains the prediction through $(t_n-t_{n-1})(\mathbf Q^{(n)}-\mathbf Z)$.

\paragraph{Alternative composition order.}
Applying the correction on the left gives
\begin{equation}
  \pi^{(n-1)}
  =D_{d_{n-1}}(\widehat\pi_0^{(n)})
    \circ\bigl(D_{d_n}(\widehat\pi_0^{(n)})\bigr)^{-1}
    \circ\pi^{(n)}.
  \label{eq:group-residual-left-update}
\end{equation}
This ordering can retain dependence on the prediction. Nevertheless, permutation sign is multiplicative and invariant under inversion, so the prediction sign cancels within each step and intermediate shuffle signs cancel across steps. Under $\pi^{(N)}=g_{d_N}$,
\begin{equation}
  \operatorname{sgn}(\pi^{(0)})
  =\operatorname{sgn}(\pi^{(N)})\operatorname{sgn}(g_{d_N})
  =+1.
  \label{eq:group-residual-parity-final}
\end{equation}
Hence this ordering cannot generate odd permutations. The limitation applies to both HA-dLLM and OS-dLLM under the stated initialization and concerns only these direct analogues, not arbitrary permutation-valued residual samplers.

%% file: files/09_soft_sampling_analysis.tex
\subsection{Analysis of Soft-State Sampling}
\label{sec:analysis-soft-sampling}

We analyze the SA-dLLM Predict--recorrupt and Residual-correction samplers defined in Equations~\eqref{eq:method-soft-predict-recorrupt} and~\eqref{eq:method-soft-residual-correction} and summarized in Algorithms~\ref{alg:sa-predict-recorrupt-sampling} and~\ref{alg:sa-residual-correction-sampling}.

\paragraph{Validity of soft-state sampling.}
Let $\mathcal{B}_K$ be the Birkhoff polytope defined in Equation~\eqref{eq:prelim-birkhoff}. For fixed $\mathbf{Z}\in\mathcal{B}_K$, the affine corruption in Equation~\eqref{eq:method-soft-forward} can be written as
$D_t^{\mathbf{Z}}(\mathbf{A})=(1-t)\mathbf{A}+t\mathbf{Z}$.
The Predict--recorrupt update in Equation~\eqref{eq:method-soft-predict-recorrupt} uses
$D_{t_{n-1}}^{\mathbf{Z}}(\mathbf{Q}^{(n)})$ directly, while
the Residual-correction update in Equation~\eqref{eq:method-soft-residual-correction} uses
\begin{align}
  \mathbf{X}_{t_{n-1}}
  &=\mathbf{X}_{t_n}
    -D_{t_n}^{\mathbf{Z}}(\mathbf{Q}^{(n)})
    +D_{t_{n-1}}^{\mathbf{Z}}(\mathbf{Q}^{(n)})\notag\\
  &=\mathbf{X}_{t_n}
    +(t_n-t_{n-1})(\mathbf{Q}^{(n)}-\mathbf{Z}).
  \label{eq:analysis-soft-residual-derivation}
\end{align}
Reusing $\mathbf{Z}$ throughout the trajectory is essential for the invariant below.

\begin{proposition}[Soft-state validity]
\label{prop:soft-state-validity}
Let $1=t_N>\cdots>t_0=0$ follow the schedule in Equation~\eqref{eq:method-soft-schedule}, initialize
$\mathbf{X}_{t_N}=\mathbf{Z}$, and suppose that
$\mathbf{Z},\mathbf{Q}^{(n)}\in\mathcal{B}_K$ at every evaluation.
Both sampling rules keep every state in $\mathcal{B}_K$. In particular, with
$\Delta_n=t_n-t_{n-1}>0$, the residual trajectory satisfies
\begin{equation}
  \mathbf{X}_{t_n}
  =t_n\mathbf{Z}
   +\sum_{m=n+1}^{N}\Delta_m\mathbf{Q}^{(m)},
  \qquad 0\leq n\leq N,
  \label{eq:analysis-soft-invariant}
\end{equation}
where an empty sum is zero. The terminal assignment is $\mathbf{A}=\mathbf{Q}^{(1)}$ for Predict--recorrupt and $\mathbf{A}=\sum_{n=1}^{N}\Delta_n\mathbf{Q}^{(n)}$ for Residual-correction.
\end{proposition}

\begin{proof}
The Predict--recorrupt update is a convex combination of $\mathbf{Q}^{(n)}$ and $\mathbf{Z}$, so feasibility follows from the convexity of $\mathcal{B}_K$. For Residual-correction, Equation~\eqref{eq:analysis-soft-invariant} holds at $n=N$ because $t_N=1$ and the sum is empty. If it holds at $n$, then
\begin{align}
  \mathbf{X}_{t_{n-1}}
  &=t_n\mathbf{Z}
    +\sum_{m=n+1}^{N}\Delta_m\mathbf{Q}^{(m)}
    +\Delta_n(\mathbf{Q}^{(n)}-\mathbf{Z})\notag\\
  &=t_{n-1}\mathbf{Z}
    +\sum_{m=n}^{N}\Delta_m\mathbf{Q}^{(m)},
\end{align}
which proves the invariant by induction. Its coefficients are nonnegative
and satisfy
\begin{equation}
  t_n+\sum_{m=n+1}^{N}\Delta_m
  =t_n+(t_N-t_n)=1.
\end{equation}
Every residual state is therefore a convex combination of doubly stochastic
matrices. Substituting $t_0=0$ gives the stated terminal assignment for Residual-correction.
\end{proof}

With finite Sinkhorn iterations, the statement holds up to normalization error. If every noise and prediction matrix is nonnegative and has row- and column-sum errors bounded by $\delta$ in maximum absolute value, then every generated state satisfies the same bound in exact arithmetic; convex combinations do not amplify these errors. The terminal Hungarian operation in Equation~\eqref{eq:method-soft-final-projection} returns a discrete permutation. Although exact Sinkhorn predictions are strictly positive, Equation~\eqref{eq:method-soft-final-scores} applies a positive floor to the terminal assignment to prevent finite-precision underflow before taking logarithms.

\paragraph{Consistency with clean-state prediction.}
If all predictions along either trajectory equal the same assignment
$\mathbf{A}_\star\in\mathcal{B}_K$, then both samplers reproduce its affine
corruption path exactly:
\begin{equation}
  \mathbf{X}_{t_n}
  =(1-t_n)\mathbf{A}_\star+t_n\mathbf{Z},
  \qquad
  \mathbf{A}=\mathbf{A}_\star.
  \label{eq:analysis-soft-consistency}
\end{equation}
For Predict--recorrupt, the result follows directly from the update. For Residual-correction, it follows from Equation~\eqref{eq:analysis-soft-invariant} and $\sum_{m=n+1}^{N}\Delta_m=1-t_n$. In the ideal case $\mathbf{A}_\star=\mathbf{P}_\star\in\mathcal{P}_K$, both samplers recover the clean permutation matrix before terminal decoding. Exact permutation-valued soft predictions are an idealized limit of finite-temperature Sinkhorn outputs.

\paragraph{Scope of the sampling guarantee.}
These results establish feasible trajectories and consistency with the model's clean-state predictions. To formalize generation, let $\omega$ collect the sampled initialization and any other randomness reused by the sampler, and let $F_{\theta,N}(c,\omega)\in\mathfrak{S}_K$ denote the final decoded output. The induced distribution is
\begin{equation}
  p_{\theta,N}^{\mathrm{alg}}(\pi\mid c)
  =\Pr_{\omega}\!\left[F_{\theta,N}(c,\omega)=\pi\right].
  \label{eq:analysis-model-induced-distribution}
\end{equation}
With fixed tie-breaking, random initialization followed by deterministic updates defines a permutation-valued generative procedure. The preceding arguments do not establish equality with the data distribution or exact sampling from reverse conditionals. The same noise matrix $\mathbf{Z}$ influences every intermediate state and subsequent prediction; consequently, each re-corrupted state is coupled to the preceding trajectory rather than being an independent sample from the forward corruption distribution conditioned on the latest prediction.